\PassOptionsToPackage{nosubfloats}{jmlrutils}

\documentclass{template}
\usepackage{graphicx}
\usepackage{booktabs,array}
\usepackage{tikz}
\usetikzlibrary{arrows.meta,decorations.markings,calc}
\usepackage{cuted}
\usepackage{caption}
\usepackage{subcaption}
\usepackage{wrapfig}
\usepackage{ragged2e}
\usepackage{bbm}
\usepackage{etoc}
\usepackage{float}

\newtheorem{assumption}{Assumption}
\newcommand{\qed}{\hfill$\Box$}

\newcommand{\eps}{\varepsilon}

\usepackage{color}

\AtBeginDocument{%
  \setlength{\parindent}{0pt}
  \setlength{\parskip}{0.5\baselineskip}
}

\AtBeginDocument{%

  \setcounter{topnumber}{4}
  \setcounter{bottomnumber}{2}
  \setcounter{totalnumber}{6}
}

\makeatletter
\AtBeginDocument{%
  \jmlrproceedings{}{}%
  \jmlrvolume{}%
  \jmlrissue{}%
  \jmlryear{}%
  \jmlrpages{}%
  \jmlrworkshop{}%
  \jmlrsubmitted{}%
  \jmlrpublished{}%
  \editor{}%

  \let\@titlefoot\@empty
  \let\ps@jmlrtps\ps@plain
}
\makeatother

\title{Statistical Vignettes on the Grokking Phenomenon}

\author{%
Baye Malick Gning$^{1,2}$\\
Elvis Dohmatob$^{1,2}$\\
    \addr{$^1$Concordia University}\\
  \addr{$^2$Mila -- Quebec AI Institute}
}

\begin{document}
\etocdepthtag.toc{main}

\maketitle
\begin{abstract}
In this work, we study grokking in a simple, analytically tractable setting whose rich phenomenology allows us to disentangle the effects of initialization, signal strength, weight decay, sample size, and input dimension. These factors can produce superficially similar learning curves through distinct mechanisms, complicating their interpretation. Our analysis distinguishes early fitting through nuisance features, core feature learning, and delayed generalization. We identify sharp conditions under which grokking provably occurs and derive analytic asymptotic expressions for the lengths of the generalization plateaus. Our results also characterize regimes of early generalization and failure to reach a prescribed test accuracy target. We show that the interpolation edge constitutes a distinct dynamical regime, where vanishing spectral gaps separate core learning from generalization, helping clarify the role of proximity to interpolation in existing accounts of grokking. The resulting picture is analogous to fast and slow phases in chemical kinetics and suggests that the lazy-to-rich transition discussed in the literature may reflect a similar separation of dynamical time scales. Our theoretical results are faithfully supported by extensive experiments.
\end{abstract}

\section{Introduction}
\label{sec:intro}
Grokking describes the delayed emergence of generalization after
a model has already achieved high training accuracy. Since its
discovery on algorithmic tasks \citep{Power2022Grokking}, it has
been associated with representation learning, regularization,
initialization, and proximity to interpolation. These perspectives
raise a common question: what determines the delay between fitting
the training data and generalizing? Answering this question requires
understanding both optimization and statistics. High training
accuracy, learning a useful feature, and attaining high test
accuracy are distinct events, and similar learning curves can
conceal different reasons for their separation.

We investigate these questions in a simple, analytically tractable
setting with a planted core feature and a model that is bilinear
on the core dimension. Despite its simplicity, this setting
exhibits rich phenomenology, allowing us to disentangle the
optimization dynamics of core learning and nuisance fitting
from their statistical consequences for generalization. Our
analysis identifies how initialization, signal strength,
regularization, and sample-to-dimension ratio jointly determine
whether and when grokking occurs. In particular, it quantifies
grokking without weight decay and reveals the distinctive role
of the interpolation edge: even when the core coefficient
approaches its desired value, residual nuisance contributions
can delay or prevent attainment of a prescribed test accuracy.
This provides a unified setting in which to examine several
mechanisms often studied separately.

\paragraph{Contributions.} Our main contributions are as follows:
\begin{itemize}
\item \emph{Quantitative separation of learning time scales.}
We derive the exact gradient flow equations for our model
(Section~\ref{sec:toy-model}) and analyze their fast and slow
dynamics using a singular perturbation perspective. We establish
a logistic approximation for core learning and quantitative
bounds on the relaxation and tracking of nuisance weights.
These results distinguish the regime below interpolation
from the interpolation edge, where a vanishing spectral gap
slows nuisance relaxation and creates a further separation
between core learning and generalization
(Section~\ref{sec:grad-flow}).

\item \emph{Conditions for grokking and plateau lengths.}
We establish a high-dimensional phase diagram separating early
generalization, grokking, and failure to attain a prescribed
test accuracy target. Where grokking provably occurs, we derive
analytic asymptotic expressions for core-learning and
test accuracy hitting times, quantifying the generalization
plateau. The resulting conditions make explicit the roles of
initialization, signal strength, regularization, and dimensionality.
They also clarify the role of weight decay: it is not necessary
for grokking, and excessive regularization can prevent attainment
of the target accuracy
(Section~\ref{sec:grok}).

\item\emph{Experimental validation.}
We verify our predictions by sampling finite training datasets and numerically integrating the corresponding gradient flow. Across sample sizes and regularization regimes, we compare core trajectories, training and test accuracies, and hitting times with their theoretical predictions. The experiments illustrate both the separation of learning time scales below
interpolation and the additional delay in generalization near the interpolation edge. See Figure \ref{fig:experiment_hard_interp}, Figure \ref{fig:experiment_underparams}, and Figure \ref{fig:plateau_scaling}.

Appendix \ref{sec:experiments} contains further empirical validation.

\end{itemize}

\begin{figure*}[!htpb]
    \centering
    \includegraphics[width=1\linewidth]{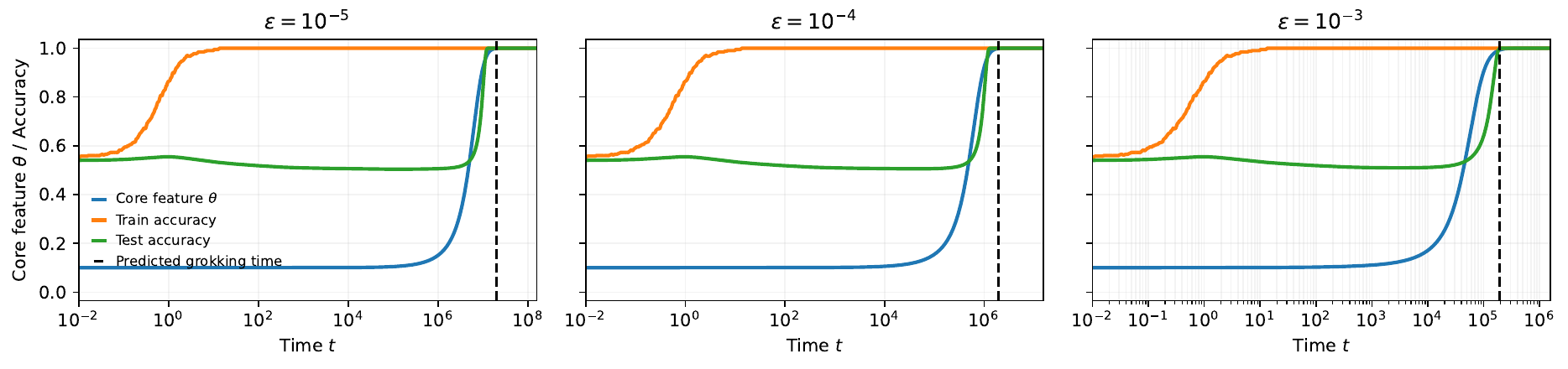}

    \vspace{-.25cm}
\caption{\textbf{Grokking near the interpolation edge ($\gamma=1$).} Experiments for the simple model \eqref{eq:toy-model} trained without regularization / weight decay ($\lambda=0$) in the optimization objective \eqref{eq:obj}, on $n=400$ samples and $d=395$ nuisance features ($d/n \approx 1$) from the data distribution \eqref{eq:distro}, for different values of the signal strength $\varepsilon$. In all cases, the effective weight decay is $b=0<1/z_p$:
the core coefficient $\theta=\varepsilon uv$ approaches $1$, and the test accuracy target $p=95\%$ is attained at a time consistent with the prediction of our theory \eqref{eq:edge-Tdelta}. Under a very small signal, note that when weight decay is active ($\lambda \in \{10^{-4},10^{-5}\}$), $b$ exceeds the threshold, preventing grokking from occurring. See Appendix \ref{sec:experiments} for details and additional experiments.}
    \label{fig:experiment_hard_interp}
\end{figure*}

\begin{figure*}[!htpb]
    \centering
    \begin{subfigure}{\textwidth}
        \centering
        \includegraphics[width=1\linewidth]{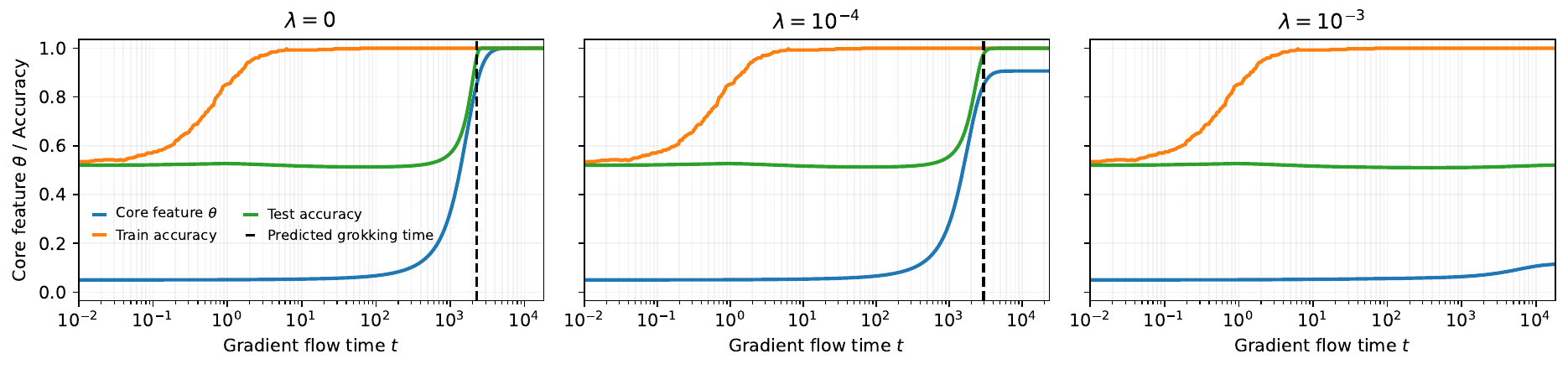}

        \caption{\textbf{Small signal strength $\varepsilon=10^{-2}$.} The vertical lines correspond to our theoretical estimate \eqref{eq:underparam-Tdelta}. For $\lambda=10^{-3}$ (rightmost plot), the effective weight decay is $\beta=1 > 1-\delta_p=0.17$, and our theory (Theorem \ref{thm:big}) correctly predicts that grokking does not occur here, contrary to the left and middle plots which correspond to values of $\beta$ smaller than the critical threshold $1-\delta_p$.}
      \end{subfigure}

      \medskip
    \begin{subfigure}{\textwidth}
        \centering      

        \includegraphics[width=1\linewidth]{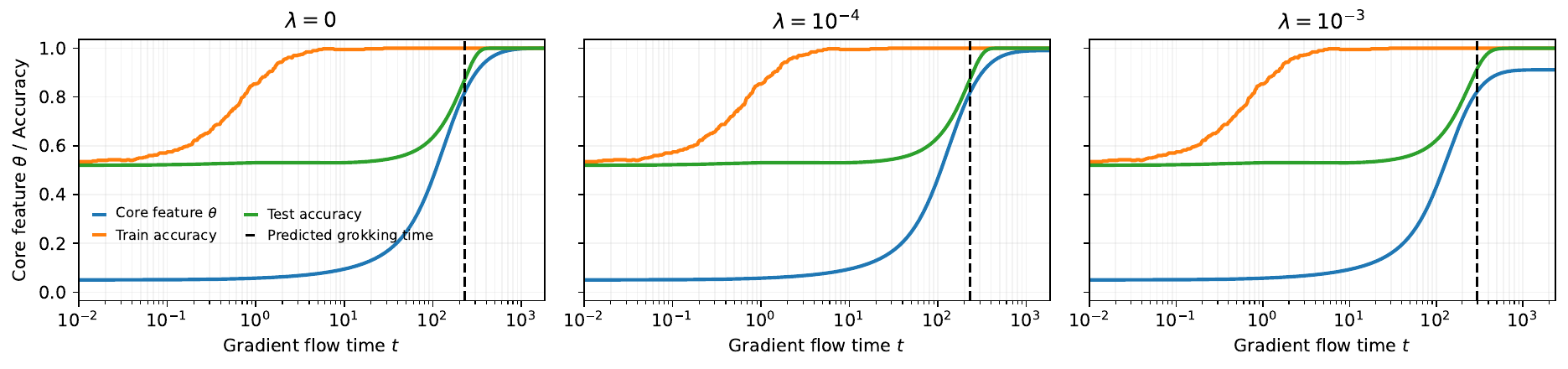}

        \caption{
        \textbf{Moderate signal strength $\varepsilon=10^{-1}$.} Here, the signal is sufficiently strong that grokking occurs even for small and moderately large values of weight decay $\lambda$. Indeed, for all values of $\lambda$ shown, the effective weight decay $\beta$ falls below the critical threshold $1-\delta_p$; our Theorem \ref{thm:big} correctly predicts that grokking occurs in time $T_p \simeq H_{\delta_p}(\beta) \ell_n$ (vertical line).}
    \end{subfigure}
\caption{\textbf{Grokking below interpolation threshold ($\gamma<1$).} Experiments use $n=400$ samples and $d=360$ nuisance features ($\gamma=0.9$). The results support Theorem~\ref{thm:big} across zero, weak and strong weight decay $\lambda$. 
The core coefficient approaches the limit $(1-\beta)_+$ predicted by our Theorem \ref{thm:big}, while weight determines whether the test accuracy target $p=0.95$ is attained. Core convergence and target attainment need not coincide. Whenever the target is reached, the observed hitting time  agrees with the theoretical prediction (vertical line). See Appendix \ref{sec:experiments} for details and more extensive experiments.
}
    \label{fig:experiment_underparams}
\end{figure*}

\section{A Minimal Setup for Grokking}
\label{sec:toy-model}

\subsection{Data Distribution}
We consider a binary classification problem with a planted
core feature and independent nuisance features. A random
input--label pair $(x,y)$ is generated as follows:
\begin{equation}
\boxed{
\begin{aligned}
    &y \sim \operatorname{Unif}(\{-1,1\}),
    &&\text{(\textbf{class label})}\\
    &x^{(1)}=\varepsilon y,
    &&\text{(\textbf{core feature})}\\
    &z\sim\mathcal N(0,I_d),\quad z\perp y,
    &&\text{(\textbf{nuisance features})}\\
    &x=(x^{(1)},z)\in\mathbb R^{d+1}.
    &&\text{(\textbf{input vector})}
\end{aligned}
}
\label{eq:distro}
\end{equation}
The parameter $\varepsilon>0$ controls the amplitude of the
core feature relative to the unit scale of each nuisance
coordinate. For every $\varepsilon>0$, the core feature
determines the label exactly, whereas the nuisance vector
is independent of the label. Thus, reducing $\varepsilon$
weakens the scale at which the signal enters the model
without increasing the Bayes classification error, which
remains zero. We are particularly interested in the
small-signal regime $\varepsilon\ll1$, where this distinction
allows us to separate the availability of predictive
information from the time required to learn to use it. A similar distribution was also considered by \citet{bayatmemorization} in the context out-of-distribution generalization.

The training dataset consists of $n$ labeled samples
$(x_1,y_1),\ldots,(x_n,y_n)$ which are iid copies of the pair $(x,y)$ defined above. Although the corresponding nuisance
features $z_i \in \mathbb R^d$ of each data point $x_i \in \mathbb R^{1+d}$ contain no information about the label $y_i$, their empirical correlations with the training
labels allow them to contribute to fitting the training data (memorization).

\subsection{A Simple Bilinear Model}
We consider  model class
\begin{align}
    f(x)=f(x;u,v,w)=uvx^{(1)}+z^\top w,
    \label{eq:toy-model}
\end{align}
with learnable parameters $u,v\in\mathbb R$ and $w\in\mathbb R^d$.
The model is linear in the input but bilinear in the parameters
associated with the core feature. This factorization $uv$ introduces
nonlinear learning dynamics while retaining a tractable
description of the competition between core learning and
nuisance fitting.

The setup is motivated by sparse parity \citep{barak2022hidden}, a task on which two-layer ReLU networks exhibit grokking
\citep{merrill2023tale}. In this problem, the label
is $y=\prod_{j\in S}x^{(j)}$ for binary inputs
$x \in\{-1,1\}^D$ and a small, unknown subset $S$.
\citet{merrill2023tale} associate delayed generalization with competition between a dense memorizing subnetwork and
a sparse generalizing subnetwork. This motivates our
separation between a predictive core and nuisance features that can fit the training data.
A parity feature admits the gating representation
$\psi_g(x)=\prod_{j=1}^D(1-g_j+g_jx^{(j)})$,
where $g\in\{0,1\}^D$ selects the participating coordinates. Learning these gates entails discovering the relevant feature. Our model isolates the subsequent competition between amplifying predictive structure and fitting nuisance features: the core is already present, but its contribution is learned through the bilinear coefficient $uv$. This simplification allows us to isolate the fine analysis of the resulting learning time scales,
while stepping aside the combinatorial problem of discovering the unknown parity mask.

\subsection{Regularized Optimization Objective and  Progress Metrics}
We train the student model \eqref{eq:toy-model} by minimizing
the regularized squared loss
\begin{equation}
\begin{aligned}
    R=R(u,v,w):=L+\frac{\lambda}{2}(u^2+v^2+\|w\|^2),
    \text{ with }L=L(u,v,w):=\frac{1}{2n}\sum_{i=1}^nr_i^2,
\end{aligned}
\label{eq:obj}
\end{equation}
where $\lambda\ge0$ is the weight-decay parameter and $r_i:=f(x_i)-y_i$.
Let $Z=(z_1,\ldots,z_n)\in\mathbb R^{n\times d}$ be the nuisance matrix with rows $z_i$ and write $Y=(y_1,\ldots,y_n)^\top\in\mathbb R^n$, the response vector.
Introducing the effective core coefficient
\begin{eqnarray}
\theta:=\varepsilon uv,
\label{eq:core}
\end{eqnarray}
we obtain $L= \|r\|^2/(2n)$, where $r:=Zw-(1-\theta)Y$ is the vector of training residuals; exact
interpolation corresponds to $r=0$, or equivalently $L=0$.
Since $x^{(1)}=\varepsilon y$, the prediction decomposes
as $f(x)=\theta y+z^\top w$. Thus, $\theta$ measures the
contribution of the planted core feature, while $z^\top w$ captures the nuisance.

Interesting, the squared $L_2$ regularization on the factors $u$ and $v$
 translate into a an $L_1$ prior on the core feature $\theta=\varepsilon u v$. Indeed, observe that for an $\theta \in \mathbb R$,
   $$
   \min_{ u,v \in \mathbb R,\, \varepsilon u v = \theta}\frac{\lambda}{2}(u^2+v^2) = \frac\lambda\varepsilon|\theta|.
   $$
For tracking learning progress during optimization,  we consider the training accuracy and test accuracy of the student model \eqref{eq:toy-model} are respectively defined by:
\begin{align*}
    \mathcal A_{{\rm te}} &:= \frac{1}{n}\sum_{i=1}^n \mathbbm{1}_{\{f(x_i) y_i > 0\}},\quad
   \mathcal A_{{\rm tr}} := \operatorname{Pr}(y f(x) > 0),
\end{align*}
for a random test point $(x,y) \sim P$ independent of the training data.
It is easy to see that
\begin{align}
   \mathcal A_{{\rm tr}} &= \frac1n\sum_{i=1}^n \mathbbm{1}_{\{1+y_ir_i \ge 0\}},\quad 
    \mathcal A_{{\rm te}} = \Phi\left(\theta/\|w\|\right)\label{eq:test-acc},
\end{align}
where $\theta:=\varepsilon uv$. 
Here, $\Phi$ is the standard normal CDF. Thus, the dynamics of the test accuracy is completely determined by the dynamics of the core feature weight $\theta$ and the nuisance norm $\|w\|$. 



\section{High-Level Overview of gradient flow Dynamics}
\label{sec:grad-flow}

In this section, we explain the geometric picture underlying the separation between early training-data fitting and delayed generalization. When the signal amplitude and weight decay are small, nuisance weights evolve rapidly while the core feature evolves slowly. Fenichel theory describes this separation through attraction toward a slow invariant manifold, followed by approximately logistic core learning. We use this picture to motivate the quantitative hitting-time analysis and highlight why approaching the interpolation edge requires additional care.

\subsection{Equations of Motion and Stationary Points}
Define the following auxiliary random quantities:
\begin{eqnarray}
\boxed{
\begin{split}
    B &:=\frac{Z^\top Z}{n} \in \mathbb R^{d \times d},\quad
    C:=B+\lambda I_d \in \mathbb R^{d \times d},\quad a:=\frac{Z^\top Y}{n} \in \mathbb R^d,\\
    \hat w &:= C^{-1}a \in \mathbb R^d,\quad
    q := 1-a^\top \hat w \in [0,1].
\end{split}
}
\label{eq:auxiliary}
\end{eqnarray}
 Also introduce the "dual" variable $\phi := \varepsilon(u^2-v^2)/2$. We analyze the evolution of the model parameters $(u,v,w)$ in \eqref{eq:toy-model} under gradient flow (GF) on the regularized objective $R$ defined in \eqref{eq:obj}. This induces continuous-time dynamics on the core variable $\theta=\varepsilon uv$.

 \begin{remark}
 We analyze gradient flow (GF) for simplicity of exposition, however our theoretical and empirical results carry effortlessly to discrete-time gradient-descent. See Appendix~\ref{sec:experiments}.
 \end{remark}

 We start with the following result.
\begin{proposition}[Equations of motion]
    The equations of motion of $(\theta,\phi,w)$ under gradient flow of the objective $R$ given in \eqref{eq:obj} are the following system of ODEs:
\begin{equation}
\label{eq:exact}
\begin{split}
 &\dot w=(1-\theta)a-C w,\quad 
 \dot\theta=2\eps h \sqrt{\theta^2 + \phi^2} -2\lambda\theta,\\
 & \dot\phi=-2\lambda\phi,\quad h=h(\theta,w):=1-\theta-a^\top w.
\end{split}
\end{equation}
\end{proposition}
    We can already see that if both $\varepsilon$ and $\lambda$ are small, then $w$ is a fast variable while $\theta$ and $\phi$ are slow variables: $w$ relaxes instantaneously compared to the clock of $\theta$ and $\phi$. This observation will ultimately be the main idea exploited to obtain a detailed account of the grokking phenomenon on this toy problem.

\begin{proposition}[Finite-sample phase boundary]
\label{prop:finite}
In terms of the $\theta$, $\phi$, and $w$ variables, the effective global minimizer of the optimization objective \eqref{eq:obj} is unique:
\begin{equation*}
\label{eq:finite-star}
\begin{split}
     \theta_*&=\left(1-\frac{\lambda}{\eps q}\right)_+,\quad
 w_*=(1-\theta_*)\hat w,\quad
 \phi_* = 0.
\end{split}
\end{equation*}
and the gradient flow \eqref{eq:exact} converges to it.

Moreover, if $\eps q>\lambda$, then there are two minima in $u/v/w$ space:
$u_*=v_*=\pm\sqrt{\theta_*/\eps}$, and one saddle point
$(u,v,w)=(0,0,\hat w)$. Every initialization verifying $u_0+v_0\ne0$
converges to a minimum point. If $\eps q\le\lambda$, the latter point is the
only stationary point and every trajectory converges to it.
\label{prop:phase-boundary}
\end{proposition}

To make our lives easy henceforth, we make the following balanced initialization assumption.
\begin{assumption}\label{eq:assumuv}
    $u(0)=v(0) \ne 0$.
\end{assumption}
Under this assumption, one has $\phi(t) = 0 < \theta(t)$ for all $t$ in the dynamics \eqref{eq:exact}. Thanks to Proposition \ref{prop:phase-boundary}, the only stationary point becomes $(\theta_*,w_*)$ defined above. Moreover, the equations of motion \eqref{eq:exact} reduce to the following coupled of ODEs:
\begin{equation}
\begin{split}
 \dot\theta &=2\eps (1-\theta - a^\top w(t)) \theta-2\lambda\theta(t),\quad 
 \dot w(t) = (1-\theta)a-Cw(t),
\end{split}
\label{eq:thetanODE}
\end{equation}
where the vector $a \in \mathbb R^d$ and matrix $C \in \mathbb R^{d \times d}$ are as defined in \eqref{eq:auxiliary}. The analysis of the above system of ODEs will be our main focus henceforth.



\subsection{Qualitative Insights: Fenichel Theory}
\label{sec:fenichel}

\begin{figure*}[!htp]
    \centering
    \begin{minipage}[t]{0.49\linewidth}
        \vspace{0pt}
        \centering
        \resizebox{\linewidth}{!}{%
\begin{tikzpicture}[
  scale=0.8,
  x={(4cm,-1.1cm)}, y={(-2.93cm,-1.45cm)}, z={(0cm,2.72cm)},
  >={Straight Barb[length=4pt,width=5pt,sep=2pt]},
  flow/.style={red,thick,postaction={decorate},
    decoration={markings,mark=at position .5 with {\arrow{#1}}}}
]
\draw[dashed,->] (-.13,0,0)--(1.52,0,0) node[below] {$\theta$};
\draw[dashed,->] (0,-.18,0)--(0,1.035,0) node[left] {$\phi$};
\draw[dashed,->] (0,0,-.52)--(0,0,1.45) node[above] {$w$};

\filldraw[fill=blue,fill opacity=.18,draw=blue!50]
  (-.10,-.38,1.10)--(1.22,-.38,-.22)--
  (1.22,1.18,-.22)--(-.10,1.18,1.10)--cycle;
\node[anchor=west,xshift=3mm] at (-.14,1.18,1.10)
  {$\mathcal M: w=(1-\theta)\hat w$};

\draw[green!60!black,thick] (1,1.18,0)--(1,-.38,0)
  node[pos=.03,sloped,above,anchor=south west,text=black]
  {$\mathcal L:\theta=1$};

\coordinate (A) at (.36,.30,.64);
\coordinate (B) at (1,.30,0);
\draw[flow={>>>}] (.36,.30,1.54)--(A);
\draw[flow={>>}] (A)--(B);
\draw[flow={>}] (1,.25,0)--(1,.04,0);

\draw ($(A)+(0,0,.06)$)--($(A)+(.03,0,.03)$)--($(A)+(.03,0,-.03)$);
\draw ($(B)+(-.03,0,.03)$)--($(B)+(-.03,.05,.03)$)--($(B)+(0,.05,0)$);

\fill (1,0,0) circle (1.5pt)
  node[above=4mm,xshift=2mm] {$(1,0,0)$};
\end{tikzpicture}
}
\end{minipage}\hfill
    \begin{minipage}[t]{0.51\linewidth}
        \vspace{0pt}
\caption{\textbf{Fast/slow dynamics of learning.} Triple arrows indicate fast dynamics, double arrows indicate slow, and single arrows indicate very slow dynamics. In the fast phase, the dynamics \eqref{eq:exact} quickly relaxes onto the plane $\mathcal M$ in constant time $T(w) \asymp O(1)$. Then a slower phase begins where the dynamics relaxes towards the line $\mathcal L \subseteq \mathcal M$, corresponding to learning the core feature weight $\theta:=\varepsilon u v$ in the student \eqref{eq:toy-model}. This happens on a time scale of order $T(\theta) \asymp 1/\varepsilon \gg 1$, where $\varepsilon \ll 1$ is the SNR. This phase is responsible for the long test accuracy, long after memorization (perfect training accuracy) has occurred. In the final / slowest phase (which only exists in the case of unbalanced initialization), the "harmonic dual" variable $\phi:=\varepsilon(u^2-v^2)/2$ relaxes to zero in time of order $T(\phi) \asymp 1/\lambda \gg T(\theta)$, where $\lambda \ll \varepsilon$ is the weight decay. 
}
\label{fig:fenichel}
    \end{minipage}
\end{figure*}
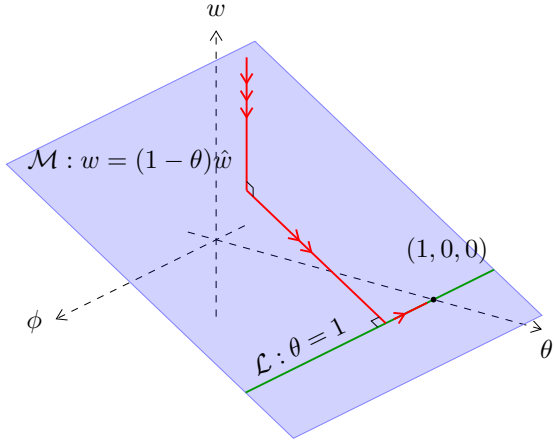
Suppose the signal strength $\varepsilon$ and weight decay $\lambda$ be small such that $\lambda/\varepsilon=O(1)$. The second equation in \eqref{eq:thetanODE} evolves faster than the first: for every value of $\theta$, $w$ instantaneously relaxes to a value satisfying the equation $(1-\theta)a-Bw=0$, which corresponds to setting the RHS of the former to zero. This implicit equation defines a so-called \emph{slow manifold} $\mathcal M$ in parameter space, on which the better part of the dynamics is trapped. Because $a:=Z^\top Y/n$ is in the image of $Z^\top$, this invariant manifold is given explicitly by:
\begin{equation}
    \mathcal M := \{(\theta,w) \in \mathbb R^{1+d} : w - (1-\theta)b \in \operatorname{Ker}Z\},
    \label{eq:slowman}
\end{equation}
where $b := B^+ a$ is min-norm nuisance interpolator. It is easy to see that $\mathcal M$ is an affine subspace of $\mathbb R^{1+d}$ of dimension $1$ if $d \le n$ and $d-n+1$ if $d > n$. 
There are thus two qualitatively distinct sub-regimes to consider. 

\textbf{The sample-poor regime: $d > n$.}
In this setting, one has $q|_{\lambda=0}=0$ almost-surely. The reduced ODEs along the slow manifold become $\mathcal M$: $\dot \theta=-2\lambda\theta$, $w_\perp = 0$, where $w_\perp := w-(1-\theta)b$ is the part of $w$ which is in the kernel of $Z$.
The dynamics converges to the pathological saddle $(0,b)$ which does not generalize: the limiting test accuracy \eqref{eq:test-acc} is $50\%$, i.e. chance level. 
    
\textbf{The sample-rich regime: $d \le n$.}
In this regime, one has $q|_{\lambda=0} > 0$ almost-surely. The transverse Jacobian of the limiting fast dynamics along $\mathcal M_0$ is
$J_0=-B$.
Since $B$ is positive definite (almost-surely), all transverse eigenvalues
are strictly negative, establishing normal attraction.
Now, consider the slow time $\tau := 2\varepsilon t$, which allows us to zoom into the slow dynamics on $\mathcal M$.  Let primes denote rate of change measured in this new clock.
Along  $\mathcal M$, the equations of motions reduce to the following  decoupled system of ODEs:
\begin{align}
\label{eq:slow-manifold-motion}
    \theta' &= q(1-\theta)\theta-\kappa\theta,\quad w=(1-\theta)b.
\end{align}
The first equation is a \emph{logistic} ODE, with a natural relation time scale of $\tau \asymp q^{-1}$, i.e
$T(\theta) \asymp (\varepsilon q)^{-1}$ physical time. Actual crossing times also depend on
initialization, regularization, and the target level.
For fixed $d<n$, Fenichel theory \citep{Fenichel1979} then suggests the following  picture:
\begin{itemize}
    \item A \emph{fast phase} during which nuisance weights approach
    a neighborhood of an attracting slow manifold. Entry within
    a fixed tolerance takes $O(1)$ physical time for a fixed
    positive spectral gap. Training accuracy can already be high during this phase, tending to perfect memorization.
    \item A \emph{slow phase}, during which the core evolves according to approximately logistic dynamics, while nuisance weights track its evolution. Whether this leads to the prescribed
    test accuracy depends on both regularization and the
    statistical regime.
\end{itemize}
Figure~\ref{fig:fenichel} illustrates this geometric picture. Note that formally, the invariant slow manifold is an $O(\varepsilon)$-perturbation $\mathcal M^\varepsilon$ of the critical manifold $\mathcal M$, rather than $\mathcal M$ itself \citep{Fenichel1979}.

This perspective is similar in spirit to the lazy-to-rich-transition account of \citet{Kumar2024LazyToRich,nishiyama2026precise}, in which early fitting precedes substantial feature learning. Our model makes a corresponding separation explicit through fast nuisance fitting and slow core learning, without identifying these phases with a neural tangent kernel transition. It also connects to \citet{Boursier2025Riemannian}, who analyze fast fitting followed by slow, weight-decay-driven norm minimization along a manifold of critical points. In our setting, small signal amplitude provides a distinct source of slow dynamics, allowing delayed generalization even without weight decay. Moreover, the statistical analysis determines when slow core learning translates into test accuracy, a distinction that becomes crucial near interpolation.

Turning the above geometric insights from Fenichel theory into a quantitative account of grokking requires comparing the core-learning scale with nuisance relaxation. The latter depends on the spectrum of the empirical covariance and need not remain bounded as the dimension grows. In particular, the proportional regime $d/n\to\gamma<1$ differs from approaching interpolation with $d<n$ but $d/n\to1$, where the spectral gap vanishes. Section \ref{sec:grok}
establishes conditions under which the separation persists and translates it into explicit hitting-time estimates.

\section{Main Results}
\label{sec:grok}
In this section, we quantify the time required to learn the core
feature $\theta$ and attain a prescribed test accuracy, and identify
when these targets are unattainable. Starting from the gradient flow
system \eqref{eq:thetanODE}, we focus on the  sample-rich setting
$d<n$ discussed at the end of Section~\ref{sec:fenichel}, where
nuisance fitting alone cannot interpolate the training labels
and core learning is necessary to drive the training loss to zero.




\subsection{Preliminaries}
    Given some quantity $S_n$ which evolves alongside the dynamics \eqref{eq:thetanODE}, and a threshold $\delta>0$, define $T_\delta(S_n)$ to be the time taken for $S_n$ to cross level $\delta$, i.e. 
\begin{eqnarray}
    T_\delta(S_n) := \inf\{t \ge 0 :  S_n(t) \ge \delta\},\text{ with }\inf\emptyset := \infty.
    \label{eq:Tdelta}
\end{eqnarray}
We shall particularly be interested in: $T_\delta(\theta_n)$, which is the time taken to learn the core feature weight above level $\delta$; $T_p(\mathcal A_{\mathrm{te},n})$, which is the time taken to reach a test accuracy of at least $100p\%$; $T_p(\mathcal A_{\mathrm{tr},n})$, which is the time  to reach a training accuracy of at least $100p\%$.

\textbf{A reference slow ODE.}
Note that hitting times  $T_\delta(S_n)$ are a random quantities, due to randomness in the training dataset which drives the dynamics. It turns out that most of the hitting times of interest to us can be neatly expressed in terms of the following deterministic function. For any core initialization $\theta_0 \in (0,1)$, level $\delta \in (\theta_0,1)$, signal strength $\varepsilon$, and scalar $\beta \ge 0$ with $\beta < \min(1-\theta_0,1-\delta)$, define 
    \begin{eqnarray}
        \quad H_\delta(\beta) := \frac{1}{1-\beta}
        \log\dfrac{(1-\theta_0 - \beta)\delta}{\theta_0\,(1-\beta-\delta)} > 0.
        \label{eq:H}
    \end{eqnarray}
Instructively, the ratio $H_\delta(\beta)$ is nothing but the time the solution $\bar\Theta(s)=\bar \Theta(s;\beta)$ of the following reference logistic ODE takes to cross a level $\delta$:
\begin{eqnarray}
\boxed{
   {\bar \Theta}'(s) = \bar \Theta(s)(1- \beta - \bar \Theta(s) ),\quad \bar \Theta(0) = \theta_0.
}
    \label{eq:bigtheta}
\end{eqnarray}
The above reference trajectory is precisely the form of the ODE on the slow manifold \eqref{eq:slow-manifold-motion} with $\kappa=q\beta$ and rescaled time $s=q \tau$. 


\textbf{Displacement from invariant manifold.}
Recall the definition of the vector $a=a_n \in \mathbb R^d$, the empirical nuisance covariance matrix $B=B_n \in \mathbb R^{d \times d}$, and the ridge nuisance interpolator $\hat w_n \in \mathbb R^d$. 
Thus, defining $e_n :=w_n-(1-\theta_n)\hat w_n$, the quantity
\begin{eqnarray}
T^\xi(\|e_n\|) = \inf\{t_0 \ge 0 : \sup_{t \ge t_0}\|e_n(t)\| \le \xi\}.    
\label{eq:entry-time}
\end{eqnarray}
measures the time the dynamics take to relax in an $\xi$-neighborhood of the slow manifold 
\begin{eqnarray}
    \mathcal M_n
:=
\left\{(\theta,w)\in\mathbb R^{1+d}:
w=(1-\theta)\hat w_n\right\}.
\label{eq:graph}
\end{eqnarray}
Along this manifold, the nuisance weights are given explicitly in terms of the core feature $\theta$. 
Note that when the regularization is turned off ($\lambda_n=0$), $\mathcal M_n$ reduces to the slow invariant manifold \eqref{eq:slowman}.


   Our analysis from Section \ref{sec:fenichel} suggested that the system of ODEs \eqref{eq:thetanODE} rapidly collapses in time $O(1)$ to the slow manifold $\mathcal M_n$. The remainder of the dynamics is then reduced to an ODE of the form \eqref{eq:slow-manifold-motion}, or \eqref{eq:bigtheta} with appropriately normalized time and "effective regularization" $\beta$.
   Thus, in this setting, one might expect that the original system relaxes in time of order $H_\delta(\beta)\ell_n$.

\textbf{Asymptotics scaling limits.} We shall work in the setting where the nuisance dimension $d = d_n$, the weight decay $\lambda = \lambda_n$, and the signal strength $\varepsilon = \varepsilon_n$  vary with sample size as follows:
\begin{equation}
\boxed{
\begin{split}
    &\kappa_n > 0,\quad \varepsilon_n > 0,\quad \lambda_n \ge 0,\quad \beta_n = O(1),\quad \|w_n(0)\| = O(1),\\
    &n\to\infty,\quad \frac{d_n}{n} \to \gamma \in (0, 1],\quad \kappa_n \to \infty,\quad \frac{\varepsilon_n}{\rho_n} \to 0,\\
   &\text{where }\kappa_n:=n-d_n,\quad \rho_n := \frac{\kappa_n}{n},\quad   \ell_n := \frac{1}{2\varepsilon_n\rho_n},\quad
   \beta_n := \frac{\lambda_n}{\varepsilon_n\rho_n}.
    \end{split}
}
        \label{eq:biglimit}
\end{equation}   
The conditions $\varepsilon_n > 0$ and $\kappa_n > 0$ are crucial: the first one ensures that there is signal to drive the learning of the core feature (albeit, possibly on an arbitrarily long time scale) while the second ensures that the nuisance design matrix $Z$ always has a trivial kernel so that the dynamics does not get trapped inside a basin of invisible nuisance features. Also, the above conditions imply
\begin{eqnarray*}
\begin{split}
    \eps_n,\lambda_n &\to 0,\, \rho_n=O(1),\, \lambda_n/\eps_n = O(1),\,
    \ell_n = \Omega(\varepsilon_n^{-1}).
\end{split}
\end{eqnarray*}
Thus, $\ell_n \to \infty$. Importantly, we allow for the possibility of no weight-decay $\lambda_n=0$. 
The case $\gamma \in (0,1)$ corresponds to the under-parametrized regime, while $\gamma=1$ is the interpolation edge where $d_n \simeq n$.

\begin{table}[htp]
\centering
\small
\renewcommand{\arraystretch}{1.2}

\begin{tabular}{
  >{\RaggedRight\arraybackslash}
  p{\dimexpr.25\linewidth-2\tabcolsep-0.75\arrayrulewidth\relax}
  >{\RaggedRight\arraybackslash}
  p{\dimexpr.25\linewidth-2\tabcolsep-0.75\arrayrulewidth\relax}
  >{\RaggedRight\arraybackslash}
  p{\dimexpr.25\linewidth-2\tabcolsep-0.75\arrayrulewidth\relax}
  >{\RaggedRight\arraybackslash}
  p{\dimexpr.28\linewidth-2\tabcolsep-0.75\arrayrulewidth\relax}
}
\toprule
\textbf{Regime}
&
\textbf{Effective dynamics \newline of core feature}
&
\textbf{Effective dynamics \newline of nuisance}
&
\textbf{Conclusion}
\\
\midrule

Under-parameterized, \newline
weak regularization
($\gamma<1$, $\beta<1-\delta_p$)
&
Logistic ODE \eqref{eq:bigtheta}; limiting core trajectory converges to $1-\beta>\delta_p$
&
Enters fixed vicinity of the slow sheet $\mathcal M_n$ in constant time, then tracks the core feature
&
Grokking occurs on time scale of order $\ell_n \ll 1$ (given in \eqref{eq:biglimit}), which diverges with the sample size $n$
\\
\midrule

Under-parameterized,\newline
strong regularization
($\gamma<1$, $\beta>1-\delta_p$)
&
Logistic ODE \eqref{eq:bigtheta};  limiting core trajectory
converges to  \newline $(1-\beta)_+<\delta_p$
&
Enters fixed vicinity of $\mathcal M_n$ in
constant time, then tracks core
&
Target test accuracy $p$ is never reached
\\
\midrule

Interpolation edge, \newline
weak regularization
($\gamma=1$, $b<1/z_p$)
&
Logistic ODE \eqref{eq:bigtheta};  limiting core trajectory
converges to $1$
&
Entry time diverges, but is of order
$o(\ell_n)$ and
$O(\rho_n^{-2}\log\rho_n^{-1})$
&
Fixed core levels reached on scale $\ell_n$ (diverging with $n$); grokking occurs on time scale
of order $\ell_n\log\rho_n^{-1} \gg \ell_n$
\\
\midrule

Interpolation edge,
strong regularization
($\gamma=1$, $b>1/z_p$)
&
Logistic ODE \eqref{eq:bigtheta}; limiting core trajectory
converges to $1$
&
Same upper-bounds as above; stationary nuisance norm $\|w_n(t)\|$ tends to $b$
&
Fixed core levels reached on scale of order  $\ell_n$,
but target test accuracy $p$ is never reached
\\
\bottomrule
\end{tabular}

\caption{\textbf{Phase diagram for the (non)emergence of grokking
(Theorem~\ref{thm:big}).}
Here, the scales $\rho_n=O(1)$ and $\ell_n\to \infty$ are
as defined in \eqref{eq:biglimit}, while $z_p$ and $\delta_p$
are as defined in \eqref{eq:thresholds}--\eqref{eq:delta}. The under-parametrized regimes additionally assume $\beta_n \to \beta$, while 
interpolation-edge regimes  assume
$\varepsilon_n\rho_n^{-2}\to0$ and
$\beta_n\sqrt{\rho_n^{-1}-1}\to b\in[0,\infty)$;
in particular, $\beta=0$ in both edge rows.
Grokking denotes early attainment of the training target
followed by delayed generalization, for the given test target $p$.
A precise hitting-time equivalent is provided in
\eqref{eq:all-timer}; the logarithmic factor at the edge
is already included in the displayed time scale.
The scale $\ell_n$ diverges with sample size $n$,
accounting for long plateaus.
Regularization  / weight decay $\lambda_n$ is classified relative
to the target $p$.
Core convergence refers to the limiting slow trajectory
as slow time tends to infinity; stationary nuisance limits
are taken as $n\to\infty$.
Limits and time estimates are in probability;
non-attainment holds with probability tending to one. Detailed results are provided in 
$T^\xi(\|e_n\|)$, defined in \eqref{eq:entry-time}.
Detailed claims are provided in Appendix \ref{sec:detailed-theory}.}
\label{tab:grokking-regimes}
\end{table}

\subsection{Delayed Generalization}

For a fixed test accuracy target $p\in(1/2,1)$, define
\begin{equation}
    z_p:=\Phi^{-1}(p),\quad
    \delta_{n,p}:=
    \frac{z_p\sqrt{\rho_n^{-1}-1}}
    {1+z_p\sqrt{\rho_n^{-1}-1}}.
    \label{eq:thresholds}
\end{equation}
Note that $z_p > 0$ and $\delta_{n,p} \in [0,1)$. Below interpolation, it holds that $ \delta_{n,p}\rightarrow
    \delta_p$ where
\begin{equation}
   \delta_p :=
    \frac{z_p\sqrt{\gamma/(1-\gamma)}}
    {1+z_p\sqrt{\gamma/(1-\gamma)}} \in (0,1),
    \label{eq:delta}
\end{equation}
whereas $\delta_{n,p}\to1$ at the interpolation edge. The following is one of our main results.
\begin{theorem}[Informal]
\label{thm:big}
For each $\gamma\in(0,1]$, there exists a threshold
$p_{\mathrm{fast}}(\gamma)\in(0,1)$ such that the following
holds. Fix a test target $p\in(p_{\mathrm{fast}}(\gamma) \land0.5,1)$.
Under \eqref{eq:biglimit} and the additional conditions
in Table~\ref{tab:grokking-regimes}, the core dynamics,
nuisance relaxation, and generalization behavior are as
stated in the table. Limits and hitting-time estimates
are understood in probability; non-attainment holds with
probability tending to one. 
\end{theorem}
Detailed statements appear as Proposition~\ref{prop:tracking}, Theorem~\ref{thm:underparam}, and Theorem~\ref{thm:edge} of Appendix~\ref{sec:detailed-theory}.

This theorem separates three events: fitting the training
data, learning the core feature, and generalizing. Training targets significantly above chance are reached in $O(1)$ time, while relaxation of the core requires the longer scale time of order $\ell_n$. Whether generalization follows on that same scale depends on dimensionality and regularization, as illustrated in
Figures~\ref{fig:experiment_hard_interp}
and~\ref{fig:experiment_underparams}.  The threshold $p_{\rm fast}(\gamma)$ is defined in equation  \eqref{eq:pfast} of the appendix,  and corresponds to the best test accuracy obtainable on a constant time scale. It is of order
$$
p_{\rm fast}(\gamma) = 0.5 +O(\theta_0)
$$
for small initialization $\theta_0 \in (0,1)$ of the core feature.

The early dynamics of training accuracy is provided in Theorem \ref{thm:early}, which establishes that every training accuracy less than an attainable ceiling $P_\gamma$ is attainable in constant time.


\paragraph{Sketch of proof of Theorem \ref{thm:big}.} One starts the first ODE in  \eqref{eq:thetanODE}, and write the following exact identity for the term involving the nuisance vector $w_n$ on the RHS:
$$
1-\theta_n(t)-a_n^\top w_n(t) = (1-a_n^\top \hat w_n)(1-\theta_n(t)) - a_n^\top e_n(t),
$$
where we recall that $e_n(t):=w_n(t) - (1-\theta_n(t))\hat w_n$ is the displacing from the slow manifold \eqref{eq:graph}, while the vector $a_n,\hat w_n \in \mathbb R^d$ and matrix $C_n \in \mathbb R^{d \times d}$ are as defined in \eqref{eq:auxiliary}.
Also, concentration gives $1-a_n^\top \hat w_n \simeq \rho_n$ in probability. The highlight of the proof of the theorem is then to show that the error term $a_n^\top e_n(t)$ is uniformly small (for $t$ sufficiently larger than a fixed constant $t_0$) and then compare the equations of motion \eqref{eq:thetanODE} with the logistic ODE \eqref{eq:bigtheta} on rescaled clock $s=t/\ell_n$, where $\ell_n$ is as given in \eqref{eq:biglimit}.

Precise statements and full details of the proof are provided in Appendix \ref{sec:detailed-theory}.

\subsection{What Controls the Plateau Length?}
\label{sec:plateau-lengths}
Under the assumptions in Theorem \ref{thm:big} ensuring delayed attainment of the test target, the plateau length is
\begin{equation}
    T_p(\mathcal A_{\mathrm{te},n})
    \simeq H_{\delta_{n,p}}(\beta_n)\ell_n,
    \label{eq:all-timer}
\end{equation}
with $H_{\delta_{n,p}}(\beta_n)$ as in \eqref{eq:H}, $\delta_{n,p}$ as in  \eqref{eq:thresholds}, and $\ell_n$ as in \eqref{eq:biglimit}. The plateau length therefore combines the base scale $\ell_n$
with the effects of initialization, regularization,
and the required core threshold.

\textbf{Below interpolation.}
Consider the regime where $\gamma < 1$ in the limits \eqref{eq:biglimit}. For $\beta<1-\delta_p$, the prediction \eqref{eq:all-timer} from our Theorem \ref{thm:big} simplifies to:
\begin{eqnarray}
            T_p(\mathcal A_{\mathrm{te},n})\simeq \ell_n H_{\delta_p}(\beta),
    \text{ with }\ell_n\simeq \frac{1}{2(1-\gamma)\varepsilon_n}.
    \label{eq:underparam-Tdelta}
\end{eqnarray}
Thus, the plateau length is inversely
proportional to signal strength $\varepsilon_n$. Weight decay also affects this length: $H_{\delta_p}(\beta)$ increases
with $\beta$ and diverges as $\beta\uparrow1-\delta_p$.
Above this threshold, there is no grokking: the target $p$ is not attained
with probability tending to $1$.

\textbf{At the interpolation edge.} Now consider the setting where $\gamma=1$ in the limits \eqref{eq:biglimit}. 
Under the additional edge assumptions, with $b<1/z_p$, observe that
$$
    (1-\delta_{n,p})/\sqrt{\rho_n} \to 1/z_p,
    \quad
    \beta_n\sqrt{\rho_n^{-1}-1}\to b.
$$
 Substitution into \eqref{eq:all-timer} and some algebra yields
\begin{eqnarray}
            T_p(\mathcal A_{\mathrm{te},n}) -  T_\delta(\theta_n) \simeq 0.5 \ell_n \log \rho_n^{-1} \gg \ell_n.
            \label{eq:edge-Tdelta}
\end{eqnarray}
Thus, in this regime, generalization requires an additional long
delay beyond learning a fixed core level. The growing
nuisance contribution must be offset by a vanishing
core deficit before the test target accuracy $p$ can be attained. 

\textbf{The effect of initialization.}
In the regimes where our theory guarantees grokking, the formulae
\eqref{eq:underparam-Tdelta}--\eqref{eq:edge-Tdelta}
describe how the plateau depends on the initialization $\theta_0$. For small $\theta_0 \in (0,1)$, the normalized
plateau length $T_p/\ell_n$ is predicted to grow logarithmically with
$1/\theta_0$, with coefficient $(1-\beta_n)^{-1}$ in the regime
below interpolation. At the interpolation edge,
$\beta_n\to0$ and the corresponding coefficient is one.
Thus, early nuisance fitting can achieve high training
accuracy while an initially weak core contribution takes
much longer to become predictive. Increasing $\theta_0$
shortens this delay. 
A similar dependence was found by \citet{nishiyama2026precise} in the context of linear regression with diagonal networks.

This dependence on $\theta_0$ gives a direct experimental diagnostic: at fixed signal amplitude, aspect ratio, and effective
regularization, plotting $T_p/\ell_n$ against
$\log(1/\theta_0)$ should approximately yield a line
of slope $(1-\beta_n)^{-1}$ for small $\theta_0$
below interpolation, and slope one at the edge.
Figure~\ref{fig:plateau_scaling} supports this prediction
in the under-parameterized regime; corresponding
edge experiments appear in Figure~\ref{fig:edgy} and Figure~\ref{fig:plateau_n=10000}
of Appendix~\ref{sec:experiments}.


\begin{figure}[!htp]
    \centering
    \begin{minipage}[c]{0.45\linewidth}
        \centering
        \includegraphics[width=\linewidth]{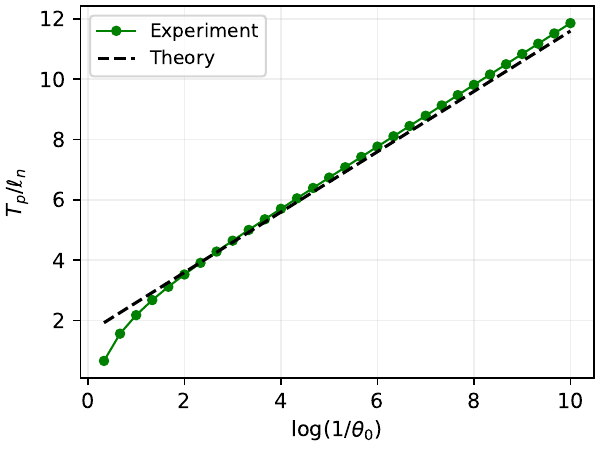}
    \end{minipage}\hfill
    \begin{minipage}[c]{0.54\linewidth}
        \caption{\textbf{Logarithmic scaling of grokking time w.r.t core feature initialization $\theta_0$}. Experiment with $n=400$ samples, nuisance dimension $d=360$, a small signal strength $\eps=10^{-2}$, and no weight decay ($\lambda=0$). Perfect terminal test accuracy ($100\%$) is always attained in this setting. Empirically measured normalized hitting time $T_p/\ell_n$ scales linearly in $\log(1/\theta_0)$ in view of \eqref{eq:underparam-Tdelta}, as predicted by Theorem \ref{thm:big}.}
        \label{fig:plateau_scaling}
    \end{minipage}
\end{figure}

\subsection{Grokking does not Require Weight Decay}

Weight decay is central to several accounts of grokking
\citep{Nanda2023ProgressMeasures,varma2023explaining,
Lyu2024Dichotomy,Boursier2025Riemannian}, but is it unclear whether it is really
necessary for the phenomenon. Setting $\lambda_n=0$
gives $\beta_n=0$, placing the dynamics in the $\beta=0$
regime below interpolation and the $b=0$ regime at the
interpolation edge. Under the corresponding assumptions, our 
Theorem~\ref{thm:big} guarantees grokking. 
This is clearly visible empirically in Figure \ref{fig:experiment_underparams} (left), Figure \ref{fig:experiment_hard_interp}, and Figure \ref{fig:plateau_scaling}.
Our results complement previous demonstrations of
unregularized grokking
\citep{Gromov2023GrokkingModularArithmetic,Kumar2024LazyToRich,LeviGrokkingEdge}
with explicit conditions and analytic plateau
lengths. They also show that regularization can modify
the delay or prevent attainment of the test target,
without being necessary for delayed generalization. 

Appendix \ref{sec:experiments} contains further empirical validation of our theoretical results.

\section{Related Work}
\label{sec:related}
In its current form, the grokking phenomenon was first observed by \citet{Power2022Grokking} as delayed
generalization following early fitting of the training data.
Subsequent work established that its occurrence depends on
initialization, regularization, data, and architecture, and that
it extends beyond algorithmic tasks \citep{Liu2023Omnigrok}.
These observations motivate studying the mechanisms behind the
delay separately from the appearance of a particular learning
curve. Our analysis distinguishes three events: attaining high
training accuracy, learning the core feature, and reaching a
prescribed test accuracy target. Their separation is essential
near interpolation, where substantial core learning need not
immediately produce generalization.

\paragraph{Plateau Phenomena and Singular Learning Theory.}
Prolonged plateaus in \emph{training error} were studied well before
grokking was documented \citep{Power2022Grokking}. Information-geometric
and statistical-mechanical analyses associated slow learning with
singular parameter configurations induced by symmetries and
redundancies, where the Fisher information becomes degenerate
\citep{Wei2008DynamicsSingularities,Amari2018DynamicsMLP}.
Classical studies of online learning in soft-committee networks
also identified long transients and examined their dependence
on teacher--student alignment \citep{saad1995line}, while later
work investigated how input spectra shape learning plateaus
\citep{yoshida2019data}. From a complementary statistical
perspective, singular learning theory relates model singularities
to generalization and marginal likelihood through algebraic
geometry \citep{watanabe2009algebraic}. These results provide
context for the slow modes and separated time scales associated
with grokking. However, training-error plateaus and delayed
generalization concern different observables: explaining grokking
requires identifying why test performance remains poor after
training accuracy has already become high.

\paragraph{Mechanistic explanations and feature learning.}
Studies of modular arithmetic have identified the representations and algorithms underlying generalization. \citet{Gromov2023GrokkingModularArithmetic} construct explicit solutions and connect grokking to the discovery of task-specific feature maps.
\citet{Nanda2023ProgressMeasures} identify continuous progress measures that distinguish memorization, circuit formation, and the removal of memorizing components. \citet{Zhong2023ClockPizza} show that changes in initialization and architecture can lead to qualitatively different algorithms for the same modular-addition task. Beyond neural networks, \citet{Mallinar2025Emergence} demonstrate grokking in recursive feature machines, relating generalization to the emergence of block-circulant structure through the average gradient outer product. These studies establish that progress toward a generalizing solution can precede its appearance in accuracy curves. Our model makes the corresponding distinction explicit through separate hitting times for the core coefficient and test accuracy.

\paragraph{Lazy-to-rich transitions and implicit bias.}
\citet{Kumar2024LazyToRich} explain grokking through a transition from an initial kernel-like regime to subsequent feature learning, emphasizing the learning rate of features and their initial alignment with the target. \citet{Lyu2024Dichotomy} prove that large initialization and small weight decay can induce a transition between early kernel bias and later minimum-norm or maximum-margin bias. For modular addition, \citet{Mohamadi2024WhyDoYouGrok} establish limitations of kernel-regime predictors and generalization guarantees for appropriately norm-controlled quadratic networks. Our fast nuisance dynamics and slow growth of the factorized core coefficient provide a tractable analogue of such staged learning. The resulting time-scale separation explains why the core can remain nearly frozen during early training. It also shows why core learning alone need not determine the onset of generalization: residual nuisance contributions can remain decisive after the core relaxes above a fixed threshold.

\paragraph{Regularization and slow geometric dynamics.}
\citet{Liu2023Omnigrok} relate grokking to the dependence of training
and test losses on weight norm.
\citet{junior2025grokking} extend regularization-based explanations
beyond the Euclidean norm, including sparsity and low-rank
structure, and show that Euclidean norm reduction is not a
universal indicator of generalization.
A particularly close theoretical connection is
\citet{Boursier2025Riemannian}, who analyze gradient flow with
vanishing weight decay. Under their assumptions, an initial
approach to a critical manifold is followed by slow Riemannian
norm minimization on the inverse-weight-decay time scale.
Our analysis also separates fast and slow dynamics, while
incorporating the statistical effects of signal strength and
sample-to-dimension ratio. In our model, delayed generalization
can persist without weight decay, and regularization determines
whether a specified test target is attainable. We establish
explicit nuisance-tracking estimates and accuracy hitting times;
near interpolation, these estimates account for the loss of a
uniform spectral gap.

\paragraph{Solvable models and interpolation thresholds.}
\citet{LeviGrokkingEdge} derive covariance-based learning
dynamics for a Gaussian linear teacher--student model and use
Marchenko--Pastur spectral analysis to characterize delayed
generalization near the interpolation edge. They also demonstrate that
sharp accuracy changes can result from the evaluation criterion,
without an accompanying discovery of a new representation.
In a distinct logistic-classification setting,
\citet{Beck2025Separabilit} identify long delays near the
linear-separability threshold, where nearly flat directions
slow optimization. Let us also mention the work of \citet{nishiyama2026precise} who propose exact mean-field theory of grokking in the setting of linear regression with so-called diagonal networks.
Our results distinguish the fixed-aspect-ratio regime below
interpolation from the interpolation edge, where the lower
covariance spectral edge vanishes. They further separate the
time needed to reach a fixed core level from the time needed
to overcome nuisance interference. This yields regimes in which
the limiting core trajectory approaches one but a prescribed
test target remains unattainable. The classification is stated
in terms of first accuracy hitting times, with explicit control
of possible earlier crossings.

\paragraph{Predicting and accelerating grokking.}
\citet{notsawo2023predicting} use oscillations in early learning
curves to predict subsequent grokking.
\citet{Lee2024Grokfast} accelerate generalization by amplifying
slowly varying gradient components.
Especially relevant to our dynamical perspective,
\citet{Pasand2026Egalitarian} identify unequal learning speeds
along principal gradient directions as a source of grokking
and propose egalitarian gradient descent to equalize those
speeds.
Our contribution complements these diagnostic and algorithmic
approaches with analytic asymptotic expressions for plateau
lengths and conditions separating early generalization,
delayed generalization, and non-attainment. The analysis also
distinguishes slow core learning from slow nuisance relaxation:
near the interpolation edge, nuisance relaxation itself need
not occur in constant physical time.

\section{Concluding Remarks}
Our tractable model disentangles the optimization and statistical
mechanisms underlying grokking. Its rich phenomenology reflects
the interplay of initialization, signal strength, regularization,
and dimensionality. Similar learning curves can conceal distinct
dynamical and statistical constraints, requiring attention to both
parameter evolution and its predictive consequences.
Separating nuisance fitting, core learning, and test accuracy
attainment yields explicit conditions for delayed generalization
and analytic asymptotic plateau lengths. At the interpolation
edge, substantial core learning need not immediately yield
generalization. Weight decay affects both the delay and
attainable accuracy, but does not dictate whether grokking occurs.

\paragraph{Limitations and future directions.}
Extending the analysis to models that must discover predictive features, such as ReLU networks learning sparse parity, is a natural
next step. This would motivate non-proportional scalings of model complexity relative to sample size, potentially revealing a richer grokking phase diagram.

\subsection*{Acknowledgment}
We would like to proudly acknowledge that the first author \textbf{Baye Malick Gning} is PhD student funded by bursaries from Canada Research Chair grant number CRC-2023-00228 (\emph{Computing for the Greater Good}) awarded to \textbf{Elvis Dohmatob}.

\subsection*{AI use statement} In this work, we have made limited use of generative AI tools for spell-checking, and also reformulating chunks of text in the introduction (Section \ref{sec:intro}) and detailed related works section (Section \ref{sec:related}). Also, AI was used to generate Tikz code for displaying the cartoon in Figure \eqref{fig:fenichel}. Finally, we used AI to better organize our proof of Theorem \ref{thm:big} in the appendix. We take responsibility for the final content of this work, including text, claims, or artifacts produced with the aid of AI.

\bibliography{refs.bib}

\appendix
\etocdepthtag.toc{appendix}

\clearpage
\thispagestyle{empty}
\vspace*{\fill}

\begin{center}
    {\Huge\bfseries Appendix}
\end{center}


\begingroup
\etocsettocstyle{\section*{}}{}
\etocsettagdepth{main}{none}
\etocsettagdepth{appendix}{subsection}
\tableofcontents
\endgroup

\section{Detailed Theorems on Early/Late Dynamics and  Plateau Length}\label{app:proofs}
The results in this section concern the high-dimensional dynamics (in the limits \eqref{eq:biglimit}) of $(\theta_n,w_n)$ defined by the system of ODEs \eqref{eq:thetanODE}. In all cases, the initialization $\theta_0\in(0,1)$ is fixed. 

\subsection{Early Dynamics of Accuracies}
Consider the scaling limit the input dimension $d = d_n$, the weight decay $\lambda = \lambda_n$, and the SNR $\varepsilon = \varepsilon_n$  vary as a function of sample size in such a way that
\begin{eqnarray}
    d_n/n \to \gamma \in (0,1],\quad \varepsilon_n \to 0,\quad \lambda_n \to 0.
    \label{eq:smalllimit}
\end{eqnarray}
In the limit $d_n/n \to \gamma (0, 1]$ which is part of \eqref{eq:biglimit}, the histogram of eigenvalues of the empirical covariance matrix $B_n:=Z^\top Z/n$ converges weakly to a Mar\v{c}enko-Pastur law \citep{MarcenkoPastur} $\mu_\gamma$. Define
\begin{align*}
    m_\gamma(t) &:= 1-\gamma + \gamma\int e^{-bt}\mu_\gamma(\mathrm db),\\
    V_\gamma(t) &:= \sigma^2 \int e^{-2bt}\mu_\gamma(\mathrm db) + \gamma h_0^2 \int \frac{(1-e^{-bt})^2}{b}\mu_\gamma(\mathrm db),\\
    D_\gamma(t) &:= h_0^2(m_\gamma(2t) - m_\gamma(t)^2) + \sigma^2 \int be^{-2bt}\mu_\gamma(\mathrm db)
\end{align*}
Also, define the following deterministic reference accuracy curves
\begin{eqnarray*}
    \begin{split}
    \mathcal A^{\mathrm {fast}}_{\mathrm{te},\gamma}(t) := \Phi\left(\frac{\theta_0}{\sqrt{V_\gamma(t)}}\right),\quad 
    \mathcal A^{\mathrm {fast}}_{\mathrm{tr},\gamma}(t) := \Phi\left(\frac{1-h_0 m_\gamma(t)}{\sqrt{D_\gamma(t)}}\right).
    \end{split}
\end{eqnarray*}
Finally, define a threshold $P_\gamma$ as follows
\begin{eqnarray}
\label{eq:Pgamma}
    P_\gamma := \begin{cases}
        \Phi\left(\dfrac{\gamma + \theta_0/(1-\theta_0)}{\sqrt{\gamma(1-\gamma)}}\right),&\mbox{if }\gamma < 1,\\
        1,&\mbox{ if }\gamma=1.
    \end{cases}
\end{eqnarray}

\begin{theorem}
\label{thm:early}
In the limit \eqref{eq:smalllimit}, the following hold.

(A) \textbf{Early memorization.} For any fixed time $t > 0$, it holds that $\mathcal A_{\mathrm{tr},n}(t) \overset{\mathbb P}{\longrightarrow} \mathcal A_{\mathrm{tr},\gamma}(t)$. Furthermore, every target training accuracy below $P_\gamma$ is attained in constant time $O_{\mathbb P}(1)$.

(B) \textbf{No early generalization.}
    For any finite time horizon $T > 0$, it holds that
    \begin{eqnarray}
    \begin{split}
                &\sup_{0 \le t \le T}|\theta_n(t)-\theta_0| \overset{\mathbb P}{\longrightarrow 0},\quad 
                \sup_{0 \le t \le T}|\mathcal A_{\mathrm{te},n}(t)-\mathcal A_{\mathrm{te},\gamma}^{\mathrm{fast}}(t)| \overset{\mathbb P}{\longrightarrow 0}
    \end{split}
    \end{eqnarray}
Moreover, for any target accuracy level $p > p_{\mathrm{fast}}(\gamma)$, it holds that $\mathbb P(T_p(\mathcal A_{\mathrm{te},n}) = \infty) \to 1$,
where
\begin{eqnarray}
 p_{\mathrm{fast}}(\gamma) :=\sup_{t\ge0}A^{\rm fast}_{\mathrm{te},\gamma}(t) = 1/2 + O(\theta_0),
 \label{eq:pfast}
\end{eqnarray}
is the best test accuracy obtainable on a constant time scale.
\end{theorem}

\begin{proof}
We prove the statement with the corrected conclusion
\[
    \Pr\!\left\{
        T_p(\mathcal A_{\mathrm{te},n})\le T
    \right\}\longrightarrow0
    \quad\text{for every fixed }T<\infty.
\]
Throughout, $h_0=1-\theta_0$, and $w_n(0)$ is independent
of the training sample and uniformly distributed on the sphere
of radius $\sigma>0$.

\paragraph{Freezing the core on bounded time intervals.}
Fix $T<\infty$. Standard Gaussian spectral bounds give
$\|B_n\|=O_{\mathbb P}(1)$. Since $\|Y_n\|=\sqrt n$,
\[
    \|a_n\|
    =\|Z_n^\top Y_n/n\|
    \le\sqrt{\|B_n\|}
    =O_{\mathbb P}(1).
\]
We first work up to a possible first crossing of $\theta_n=2$.
Variation of constants and positivity of $B_n+\lambda_nI$
give, on this interval,
\[
    \sup_{t\le T}\|w_n(t)\|
    \le \sigma+T\|a_n\|.
\]
The core equation consequently implies
\[
    \sup_{t\le T}|\dot\theta_n(t)|
    =O_{\mathbb P}(\varepsilon_n+\lambda_n)
\]
up to that crossing. Since $\theta_0<1$, integration excludes
a crossing before $T$ with probability tending to one.
Therefore
\begin{equation}
    \sup_{t\le T}|\theta_n(t)-\theta_0|
    =O_{\mathbb P}(\varepsilon_n+\lambda_n)
    \longrightarrow0.
    \label{eq:fast-core-freezing}
\end{equation}

Let $v_n$ be the nuisance trajectory with the core frozen at
$\theta_0$ and with zero weight decay:
\[
    \dot v_n=h_0a_n-B_nv_n,\quad v_n(0)=w_n(0).
\]
Subtracting the two nuisance equations gives
\[
    w_n(t)-v_n(t)
    =-\int_0^t e^{-(B_n+\lambda_nI)(t-s)}
       \left[
           a_n(\theta_n(s)-\theta_0)+\lambda_nv_n(s)
       \right]\mathrm ds.
\]
Since $\sup_{s\le T}\|v_n(s)\|\le\sigma+T\|a_n\|$,
\begin{equation}
    \sup_{t\le T}\|w_n(t)-v_n(t)\|
    \overset{\mathbb P}{\longrightarrow}0.
    \label{eq:fast-nuisance-freezing}
\end{equation}

\paragraph{The limiting nuisance norm.}
Define
\[
    F_t(b)=
    \begin{cases}
        (1-e^{-bt})/b,&b>0,\\
        t,&b=0.
    \end{cases}
\]
The frozen trajectory is
\[
    v_n(t)=e^{-B_nt}w_n(0)+h_0F_t(B_n)a_n.
\]
Expanding its squared norm gives three terms.

First, spherical quadratic-form concentration yields, for fixed $t$,
\[
    \|e^{-B_nt}w_n(0)\|^2
    -\frac{\sigma^2}{d_n}\operatorname{tr}(e^{-2B_nt})
    \overset{\mathbb P}{\longrightarrow}0.
\]
Indeed, conditional on $B_n$, the variance is bounded by
$2\sigma^4\|e^{-2B_nt}\|^2/d_n$.

For the second term, put
$A_n(t)=Z_nF_t(B_n)^2Z_n^\top/n^2$. Then
\[
    \|F_t(B_n)a_n\|^2=Y_n^\top A_n(t)Y_n,
\]
and, conditional on $Z_n$,
\[
\begin{aligned}
    \mathbb E[Y_n^\top A_n(t)Y_n\mid Z_n]
        &=\frac1n\operatorname{tr}(B_nF_t(B_n)^2),\\
    \operatorname{Var}(Y_n^\top A_n(t)Y_n\mid Z_n)
        &\le2\operatorname{tr}(A_n(t)^2).
\end{aligned}
\]
On events where $\|B_n\|$ is bounded, the last quantity
is $O(1/n)$ for fixed $t$. Thus
\[
    \|F_t(B_n)a_n\|^2
    -\frac1n\operatorname{tr}(B_nF_t(B_n)^2)
    \overset{\mathbb P}{\longrightarrow}0.
\]
Finally, the cross term vanishes in probability. Conditional
on the sample, it is a centered linear functional of $w_n(0)$,
with variance bounded by
\[
    \frac{\sigma^2}{d_n}
    \|e^{-B_nt}F_t(B_n)a_n\|^2
    =O_{\mathbb P}(1/d_n).
\]

Applying Marchenko--Pastur convergence to these trace expressions
proves
\[
    \|v_n(t)\|^2\overset{\mathbb P}{\longrightarrow}V_\gamma(t)
\]
at every fixed $t$. This convergence is uniform on $[0,T]$.
To see this explicitly, the bounds on $B_n$, $a_n$, and $v_n$
imply
\[
    \sup_{t\le T}
    \left|\frac{\mathrm d}{\mathrm dt}\|v_n(t)\|^2\right|
    =O_{\mathbb P}(1).
\]
Given an error tolerance, first bound this Lipschitz constant
with arbitrarily high probability, then choose a sufficiently
fine deterministic grid in $[0,T]$. Convergence at the finitely
many grid points, together with continuity of $V_\gamma$,
controls the supremum between grid points. Hence
\begin{equation}
    \sup_{t\le T}
    \left|\|w_n(t)\|^2-V_\gamma(t)\right|
    \overset{\mathbb P}{\longrightarrow}0,
    \label{eq:fast-norm-convergence}
\end{equation}
where we also used \eqref{eq:fast-nuisance-freezing}.

Because $\sigma>0$,
\[
    \inf_{0\le t\le T}V_\gamma(t)
    \ge \sigma^2\int e^{-2bT}\mu_\gamma(\mathrm db)>0.
\]
The exact test accuracy identity, together with
\eqref{eq:fast-core-freezing} and
\eqref{eq:fast-norm-convergence}, now gives
\[
    \sup_{t\le T}
    \left|
        \mathcal A_{\mathrm{te},n}(t)
        -\mathcal A^{\mathrm{fast}}_{\mathrm{te},\gamma}(t)
    \right|
    \overset{\mathbb P}{\longrightarrow}0.
\]

\paragraph{Training margins for the frozen dynamics.}
It remains to analyze training accuracy, which depends on
the signs of individual margins. Absorb the labels into
the design by setting
\[
    X_n=\operatorname{diag}(Y_n)Z_n,\quad
    \mathbf1=(1,\ldots,1)^\top,\quad
    \kappa_n=X_nX_n^\top/n.
\]
The entries of $X_n$ are again independent standard Gaussians,
and $X_n^\top X_n/n=B_n$. The vector of frozen training
margins is
\[
    M_n(t):=\theta_0\mathbf1+X_nv_n(t)
    =\mathbf1-h_0e^{-\kappa_nt}\mathbf1+\xi_n(t),
\]
where
\[
    \xi_n(t)=X_ne^{-B_nt}w_n(0).
\]
Let
\[
    \overline M_n(t)=\frac{\mathbf1^\top M_n(t)}n,
    \quad
    R_n(t)=M_n(t)-\overline M_n(t)\mathbf1.
\]

We claim that, for fixed $t$,
\begin{equation}
    \overline M_n(t)\to1-h_0m_\gamma(t),
    \quad
    \frac{\|R_n(t)\|^2}{n}\to D_\gamma(t)
    \quad\text{in probability}.
    \label{eq:fast-margin-moments}
\end{equation}
To verify this, write $u_n=\mathbf1/\sqrt n$.
Orthogonal invariance of $\kappa_n$ and spherical
quadratic-form concentration imply
\[
    u_n^\top e^{-j\kappa_nt}u_n
    -\frac1n\operatorname{tr}(e^{-j\kappa_nt})
    \longrightarrow0,\quad j=1,2.
\]
The trace limits are $m_\gamma(jt)$: the $n$-dimensional
Gram matrix has a fraction $1-d_n/n$ of zero eigenvalues,
and its remaining eigenvalues are those of $B_n$.

Conditional on $X_n$, the vector $w_n(0)$ is centered
and isotropic. In particular,
\[
    \frac{\mathbf1^\top\xi_n(t)}n\longrightarrow0,
\]
and spherical quadratic-form concentration gives
\[
    \frac{\|\xi_n(t)\|^2}{n}
    =w_n(0)^\top B_ne^{-2B_nt}w_n(0)
    \longrightarrow
    \sigma^2\int b e^{-2bt}\mu_\gamma(\mathrm db).
\]
The normalized cross term between
$e^{-\kappa_nt}\mathbf1$ and $\xi_n(t)$ also tends to zero.
Indeed, conditional on $X_n$, it is centered and its variance
is bounded by a constant times $1/d_n$ on events where
$\|B_n\|$ is bounded. Expanding
$\|M_n-\overline M_n\mathbf1\|^2/n$ now proves
\eqref{eq:fast-margin-moments}.

\paragraph{From margin moments to training accuracy.}
Moments alone would not determine the fraction of positive
margins. Here Gaussian rotational symmetry supplies the
additional information.

For any orthogonal matrix $O$ with $O\mathbf1=\mathbf1$,
the matrices $OX_n$ and $X_n$ have the same distribution,
and replacing $X_n$ by $OX_n$ replaces $M_n(t)$ by $OM_n(t)$.
Consequently, conditional on $\overline M_n(t)$ and
$\|R_n(t)\|$, the direction of $R_n(t)$ is uniform on the
unit sphere in $\mathbf1^\perp$. We may therefore represent
its conditional distribution using independent standard
Gaussian variables $G_1,\ldots,G_n$:
\[
    M_{n,i}(t)
    \stackrel{\mathrm d}{=}
    \overline M_n(t)
    +\frac{\|R_n(t)\|}{\sqrt n}
      \frac{G_i-\overline G}
      {\left[n^{-1}\sum_j(G_j-\overline G)^2\right]^{1/2}},
\]
where the Gaussian vector is independent of the two
conditioning quantities.

The law of large numbers gives
$\overline G\to0$ and
$n^{-1}\sum_j(G_j-\overline G)^2\to1$.
Together with \eqref{eq:fast-margin-moments} and convergence
of the empirical Gaussian distribution, this proves
\[
    \frac1n\sum_{i=1}^n\mathbf1\{M_{n,i}(t)>0\}
    \overset{\mathbb P}{\longrightarrow}
    \Phi\!\left(
        \frac{1-h_0m_\gamma(t)}{\sqrt{D_\gamma(t)}}
    \right).
\]
The denominator is positive at every finite $t$, since
$\sigma>0$ and
$\int b e^{-2bt}\mu_\gamma(\mathrm db)>0$.

To transfer this conclusion to the original dynamics, let
$M_n^{\mathrm{actual}}(t)$ denote its margin vector. Then
\[
\begin{aligned}
    \frac1n\|M_n^{\mathrm{actual}}(t)-M_n(t)\|^2
    &\le
    2|\theta_n(t)-\theta_0|^2
    +2\|B_n\|\|w_n(t)-v_n(t)\|^2\\
    &\overset{\mathbb P}{\longrightarrow}0.
\end{aligned}
\]
For any $a>0$, the fraction of signs that differ is at most
\[
    \frac1n\sum_i\mathbf1\{|M_{n,i}(t)|\le a\}
    +\frac{\|M_n^{\mathrm{actual}}(t)-M_n(t)\|^2}{na^2}.
\]
The first term converges to the probability that a
nondegenerate Gaussian lies in $[-a,a]$; the second vanishes.
Letting $a\downarrow0$ proves
\[
    \mathcal A_{\mathrm{tr},n}(t)
    \overset{\mathbb P}{\longrightarrow}
    \mathcal A^{\mathrm{fast}}_{\mathrm{tr},\gamma}(t).
\]

\paragraph{Training targets and exclusion of early test crossings.}
For $\gamma<1$, dominated convergence gives
\[
    m_\gamma(t)\to1-\gamma,\quad
    D_\gamma(t)\to h_0^2\gamma(1-\gamma).
\]
Thus
$\mathcal A^{\mathrm{fast}}_{\mathrm{tr},\gamma}(t)\to P_\gamma$.
For $\gamma=1$, the Marchenko--Pastur law has no atom at zero,
so $m_1(t)\to0$ and $D_1(t)\to0$, while the numerator of
the training-accuracy formula tends to one. Hence the
limiting training accuracy tends to $P_1=1$.

Given a training target $q<P_\gamma$, choose a finite
deterministic $t_q$ with
$\mathcal A^{\mathrm{fast}}_{\mathrm{tr},\gamma}(t_q)>q$.
Fixed-time convergence implies
\[
    \Pr\{T_q(\mathcal A_{\mathrm{tr},n})\le t_q\}\to1,
\]
establishing the claimed $O_{\mathbb P}(1)$ bound.

Finally, fix $p>p_{\mathrm{fast}}(\gamma)$ and $T<\infty$.
If the actual test accuracy reaches $p$ before $T$, then
\[
    \sup_{t\le T}
    \left|
        \mathcal A_{\mathrm{te},n}(t)
        -\mathcal A^{\mathrm{fast}}_{\mathrm{te},\gamma}(t)
    \right|
    \ge p-p_{\mathrm{fast}}(\gamma)>0.
\]
Uniform convergence makes the probability of this event
tend to zero. This proves exclusion of early attainment,
while allowing a crossing on a diverging time scale.
\end{proof}

\subsection{Generalization: Convergence and Plateau Lengths}
\label{sec:detailed-theory}
The results in this section taken together, make the statement of Theorem \ref{prop:finite} of the main text precise.

\begin{proposition}[Core learning and nuisance tracking]
\label{prop:tracking}
Define
\[
    \hat w_n=(B_n+\lambda_n I)^{-1}a_n,
    \quad
    c_n=\frac{1-a_n^\top\hat w_n}{\rho_n},
\]
where $\rho_n$ is as in \eqref{eq:biglimit}, and let $L_n$ solve the logistic ODE
\[
    L_n'=L_n(c_n-\beta_n-c_nL_n),
    \quad L_n(0)=\theta_0.
\]

\textbf{(A) Convergence of core feature trajectory.} In the limit \eqref{eq:biglimit}, it holds that
\begin{align}
        c_n\overset{\mathbb P}{\longrightarrow}1,
    \quad
    \frac{\|\hat w_n\|^2}{\rho_n^{-1}-1}
    &\overset{\mathbb P}{\longrightarrow}1,\quad 
    \sup_{s\ge0}
    |\theta_n(\ell_ns)-L_n(s)|
    =O_{\mathbb P}(\varepsilon_n/\rho_n).
\end{align}
In particular, for every $S<\infty$, it holds that
\begin{eqnarray}
        \sup_{0\le s\le S}
    |\theta_n(\ell_ns)-\bar\Theta(s)|
    \overset{\mathbb P}{\longrightarrow}0,
\end{eqnarray}
where $\bar \Theta$ is as defined by the reference slow ODE \eqref{eq:bigtheta}.

\textbf{(B) Core feature plateau length.} For every fixed $\delta\in(\theta_0,1)$, it holds in the limit \eqref{eq:biglimit} that
\begin{eqnarray}
    \begin{cases}
    T_\delta(\theta_n)
    \simeq_{\mathbb P}\ell_nH_\delta(\beta),
    & \beta<1-\delta,\\[3pt]
    \Pr(T_\delta(\theta_n)=\infty)\to1,
    & \beta>1-\delta,
\end{cases}
\end{eqnarray}
where the function $H_\delta$ is as defined in \eqref{eq:H}.
Furthermore, if $\beta=1-\delta$, then
$T_\delta(\theta_n)/\ell_n\to\infty$ in probability.

\textbf{(C) Slow manifold Entry.} For $e_n(t)=w_n(t)-(1-\theta_n(t))\hat w_n$, there
exist deterministic constants $C,c>0$ such that, in the limit \eqref{eq:biglimit},
with probability tending to one,
\[
    \|e_n(t)\|
    \le C\rho_n^{-1}e^{-c\rho_n^2t}
       +C\varepsilon_n\rho_n^{-2}
    \quad\text{for all }t\ge0.
\]
If additionally $\varepsilon_n/\rho_n^2\to0$, then,
for every fixed $s_0>0$,
\[
    \sup_{s\ge s_0}
    \|w_n(\ell_ns)-(1-L_n(s))\hat w_n\|
    \overset{\mathbb P}{\longrightarrow}0.
\]
\end{proposition}

\begin{theorem}[Below interpolation]
\label{thm:underparam}
Consider the limit \eqref{eq:biglimit} with $\gamma\in(0,1)$. Then, for every $0<s_0<S<\infty$, it holds that
\begin{align}
        \sup_{s\in[s_0,S]}
    \left|
        \|w_n(\ell_ns)\|^2
        -\frac{\gamma}{1-\gamma}(1-\bar\Theta(s))^2
    \right|
    &\overset{\mathbb P}{\longrightarrow}0,\\
    \sup_{s\in[s_0,S]}
    \left|
        \mathcal A_{\mathrm{te},n}(\ell_ns)
        -\Phi\!\left(
            \frac{\bar\Theta(s)}
            {(1-\bar\Theta(s))\sqrt{\gamma/(1-\gamma)}}
        \right)
    \right|
    &\overset{\mathbb P}{\longrightarrow}0.
\end{align}
Furthermore, for every fixed $\xi>0$, it holds that
    $T^\xi(\|e_n\|)=O_{\mathbb P}(1)$

Finally, for every target test accuracy $p\in(p_{\mathrm{fast}}(\gamma),1)$, it holds that
\begin{eqnarray}
    \begin{cases}
    T_p(\mathcal A_{\mathrm{te},n})
    \simeq_{\mathbb P}\ell_nH_{\delta_p}(\beta),
    & \beta<1-\delta_p,\\[3pt]
    \Pr(T_p(\mathcal A_{\mathrm{te},n})=\infty)\to1,
    & \beta>1-\delta_p.
\end{cases}
\end{eqnarray}
The equality case is excluded.
\end{theorem}

\begin{theorem}[At the interpolation edge]
\label{thm:edge}
Consider the limits \eqref{eq:biglimit} with  $\gamma=1$, and additionally assume the following:
\begin{eqnarray}
\varepsilon_n/\rho_n^2\to0,
    \quad
    \beta_n\sqrt{\rho_n^{-1}-1}\to b\in[0,\infty).
\end{eqnarray}
Then, $\beta_n \to \beta = 0$ and the following hold:
\begin{itemize}
\item The reference core trajectory \eqref{eq:bigtheta} satisfies
$\bar\Theta(s)\to1$ as $s\to\infty$.
\item Every fixed
$\delta\in(\theta_0,1)$ satisfies
\begin{eqnarray}
        T_\delta(\theta_n)
    \simeq_{\mathbb P}\ell_nH_\delta(0).
\end{eqnarray}
\end{itemize}

Furthermore, for every fixed $\xi>0$, it holds that
\begin{align}
        T^\xi(\|e_n\|)
    &\overset{\mathbb P}{\longrightarrow}\infty,\\
    T^\xi(\|e_n\|)
    &=O_{\mathbb P}\!\left(
        \rho_n^{-2}\log\rho_n^{-1})
    \right)
    =o_{\mathbb P}(\ell_n),
\end{align}
and the stationary limits satisfy
\begin{eqnarray}
        \theta_n(\infty)
    \overset{\mathbb P}{\longrightarrow}1,
    \quad
    \|w_n(\infty)\|
    \overset{\mathbb P}{\longrightarrow}b.
\end{eqnarray}

Finally, for every target accuracy level $p\in(p_{\mathrm{fast}}(1),1)$, it holds that
\begin{eqnarray}
    \begin{cases}
    T_p(\mathcal A_{\mathrm{te},n})
    \simeq_{\mathbb P}\ell_nH_{\delta_{n,p}}(\beta_n)
    \simeq_{\mathbb P}\dfrac{\ell_n}{2}\log\dfrac{1}{\rho_n},
    & b<z_p^{-1},\\[3pt]
    \Pr(T_p(\mathcal A_{\mathrm{te},n})=\infty)\to1,
    & b>z_p^{-1}.
\end{cases}
\end{eqnarray}
\end{theorem}

\section{Proofs of the Main Results}
\label{app:main-proofs}

Throughout, convergence and stochastic order estimates are with
respect to the training sample and initialization. We use the
global existence and finite-sample convergence result established
earlier, together with the fast-phase accuracy convergence.
The random-matrix inputs used below are stated explicitly.

\subsection{Proof of Proposition~\ref{prop:tracking}}
All asymptotics are w.r.t the limits \eqref{eq:biglimit}. Write
\[
    C_n=B_n+\lambda_nI,\quad
    q_n=1-a_n^\top C_n^{-1}a_n,\quad
    g_n=\|C_n^{-1}a_n\|^2,
\]
so that $\hat w_n=C_n^{-1}a_n$ and $c_n=q_n/\rho_n$.

\paragraph{Sample estimates.}
We first record the Gaussian estimates needed in the proof.
At zero regularization, let
\[
    q_{0,n}=1-a_n^\top B_n^{-1}a_n,
    \quad
    g_{0,n}=\|B_n^{-1}a_n\|^2.
\]
The random column space of $Z_n$ is uniformly distributed and
independent of $Y_n$. Consequently,
\begin{align*}
        q_{0,n}\sim
    \operatorname{Beta}\!\left(\kappa_n/2,d_n/2\right),
    \quad
    \mathbb Eq_{0,n}=\rho_n,\quad
    \operatorname{Var}\!\left(q_{0,n}/\rho_n\right)
   =\frac{2d_n}{(n+2)\kappa_n}\longrightarrow0.
\end{align*}
Thus $q_{0,n}/\rho_n\to1$ in probability.

For completeness, the corresponding inverse-covariance estimate is
\[
    \frac{g_{0,n}}{d_n/\kappa_n}
    \overset{\mathbb P}{\longrightarrow}1.
\]
It follows from the Gaussian polar decomposition
$Z_n=U_nW_n^{1/2}$, where $W_n=Z_n^\top Z_n$ and $U_n$ is independent
of $W_n$. Indeed, $g_{0,n}=Y_n^\top U_nW_n^{-1}U_n^\top Y_n$,
and spherical second and fourth moments give
\[
\begin{aligned}
    \mathbb E[g_{0,n}\mid W_n]
        &=\operatorname{tr}(W_n^{-1}),\quad 
    \operatorname{Var}(g_{0,n}\mid W_n)
        =\frac{2}{n+2}\left[
            n\operatorname{tr}(W_n^{-2})
            -\operatorname{tr}(W_n^{-1})^2
        \right].
\end{aligned}
\]
The inverse-Wishart moment identities
\[
\begin{aligned}
    \mathbb E\operatorname{tr}(W_n^{-1})
        &=\frac{d_n}{\kappa_n-1},\quad
    \mathbb E\operatorname{tr}(W_n^{-2})
        =\frac{d_n(n-1)}{\kappa_n(\kappa_n-1)(\kappa_n-3)},\quad
    \operatorname{Var}\!\left(\operatorname{tr}(W_n^{-1})\right)
        =\frac{2d_n(n-1)}
        {\kappa_n(\kappa_n-1)^2(\kappa_n-3)}
\end{aligned}
\]
then imply the assertion by Chebyshev's inequality, since
$\kappa_n\to\infty$.

We also use the standard extreme-singular-value estimates for
rectangular Gaussian matrices: there are deterministic constants
$c,C>0$ such that
\begin{equation}
    \Pr\!\left\{
        \lambda_{\min}(B_n)\ge c\rho_n^2,\quad
        \|B_n\|\le C
    \right\}\longrightarrow1.
    \label{eq:proof-spectral-bounds}
\end{equation}
In particular,
\[
    \lambda_n/\lambda_{\min}(B_n)
    =O_{\mathbb P}(\varepsilon_n/\rho_n)\longrightarrow0.
\]
Comparison in an eigenbasis of $B_n$ therefore gives
$g_n/g_{0,n}\to1$. Moreover,
\[
    0\le q_n-q_{0,n}\le\lambda_n\,g_{0,n},
\]
and hence
\[
    0\le\frac{q_n-q_{0,n}}{\rho_n}
    \le\beta_n\,\varepsilon_n\,g_{0,n}
    =O_{\mathbb P}(\varepsilon_n/\rho_n)\longrightarrow0.
\]
This proves
\[
    c_n\to1,\quad
    \frac{g_n}{\rho_n^{-1}-1}\to1.
\]

\paragraph{A scalar identity for the core.}
The nuisance equation implies
\[
    \hat w_n^\top\dot w_n
    =(1-\theta_n)a_n^\top\hat w_n-a_n^\top w_n.
\]
Consequently,
\[
    1-\theta_n-a_n^\top w_n
    =q_n(1-\theta_n)+\hat w_n^\top\dot w_n.
\]
Set $\Theta_n(s)=\theta_n(\ell_ns)$. Integrating the equation
for $\log\theta_n$ gives the exact identity
\begin{equation}
    \log\frac{\Theta_n(s)}{\theta_0}
    =\int_0^s[c_n(1-\Theta_n(r))-\beta_n]\,\mathrm dr
      +E_n(s),
    \label{eq:proof-log-identity}
\end{equation}
where
\[
    E_n(s)=2\varepsilon_n\hat w_n^\top
    [w_n(\ell_ns)-w_n(0)].
\]

We bound this error before a possible first crossing of
$\theta_n=2$. On that interval, $|1-\theta_n|\le1$.
Variation of constants in an eigenbasis of $C_n$ yields
\[
    \|w_n(t)\|\le\|w_n(0)\|+\|\hat w_n\|,
\]
because each coordinate of the forced response has magnitude
at most the corresponding coordinate magnitude of $\hat w_n$.
Thus
\[
    |E_n(s)|\le D_n:=
    2\varepsilon_n\|\hat w_n\|
    (2\|w_n(0)\|+\|\hat w_n\|)
    =O_{\mathbb P}(\varepsilon_n/\rho_n).
\]

Let $X_n=\Theta_ne^{-E_n}$. Differentiating
\eqref{eq:proof-log-identity} gives
\[
    X_n'=X_n(c_n-\beta_n-c_ne^{E_n}X_n).
\]
Scalar comparison with the logistic equation defining $L_n$
shows that, while $|E_n|\le D_n$,
\[
    e^{-D_n}L_n(s)\le X_n(s)\le e^{D_n}L_n(s),
\]
and therefore
\begin{equation}
    e^{-2D_n}L_n(s)\le\Theta_n(s)\le e^{2D_n}L_n(s).
    \label{eq:proof-core-sandwich}
\end{equation}
Here $0<L_n(s)\le1$. On the event $e^{2D_n}<2$, the upper
bound prevents $\Theta_n$ from reaching $2$. The preceding
estimates thus hold for all $s\ge0$ on an event whose
probability tends to one. It follows that
\[
    \sup_{s\ge0}|\Theta_n(s)-L_n(s)|
    \le e^{2D_n}-1
    =O_{\mathbb P}(\varepsilon_n/\rho_n).
\]
Since $c_n\to1$ and $\beta_n\to\beta$, continuous dependence
of the scalar logistic equation gives the stated convergence
to $\bar\Theta$ on bounded slow-time intervals.

\paragraph{Core hitting times.}
If $\beta<1-\delta$, the limiting trajectory crosses $\delta$
strictly increasingly at $H_\delta(\beta)$. For every sufficiently
small $h>0$,
\[
    \max_{0\le s\le H_\delta(\beta)-h}\bar\Theta(s)<\delta,
    \quad
    \bar\Theta(H_\delta(\beta)+h)>\delta.
\]
Uniform convergence on this interval implies, with probability
tending to one,
\[
    H_\delta(\beta)-h
    <T_\delta(\theta_n)/\ell_n
    <H_\delta(\beta)+h.
\]
Letting $h\downarrow0$ proves the equivalent.

If $\beta>1-\delta$, then the quantity
\[
    \sup_{s\ge0}L_n(s)
    =\max\!\left\{\theta_0,
        \left(1-\beta_n/c_n\right)_+\right\}
\]
converges to a value strictly below $\delta$.
The global bound \eqref{eq:proof-core-sandwich} excludes any
crossing with probability tending to one. At equality,
$\bar\Theta(s)<\delta$ on every bounded interval, which proves
$T_\delta(\theta_n)/\ell_n\to\infty$.

\paragraph{Nuisance tracking.}
Direct differentiation of
$e_n=w_n-(1-\theta_n)\hat w_n$ gives
\[
\begin{aligned}
    \dot e_n
    ={}&-(C_n+2\varepsilon_n\theta_n\hat w_na_n^\top)e_n+2\varepsilon_n\theta_n\hat w_n
       [q_n(1-\theta_n)-\lambda_n/\varepsilon_n].
\end{aligned}
\]
Put $z_n=C_n^{1/2}e_n$ and $u_n=C_n^{-1/2}a_n$. Then
\[
\begin{aligned}
    \dot z_n
    ={}&-(C_n+2\varepsilon_n\theta_nu_nu_n^\top)z_n +2\varepsilon_n\theta_nu_n
       [q_n(1-\theta_n)-\lambda_n/\varepsilon_n].
\end{aligned}
\]
The matrix in the first line is symmetric and bounded below
by $C_n$. Also,
\[
    \|u_n\|^2=a_n^\top C_n^{-1}a_n=1-q_n\le1.
\]
On events of probability tending to one, the preceding
estimates give $\theta_n\le2$, $q_n\le C\rho_n$,
$\lambda_n/\varepsilon_n\le C\rho_n$, and $\|z_n(0)\|\le C$.
Writing $m_n=\lambda_{\min}(C_n)$, the norm inequality yields
\[
    \|z_n(t)\|
    \le Ce^{-m_nt}
       +C\varepsilon_n\rho_n/m_n.
\]
Using $m_n\ge c\rho_n^2$ and
$\|e_n\|\le m_n^{-1/2}\|z_n\|$ proves
\[
    \|e_n(t)\|
    \le C\rho_n^{-1}e^{-c\rho_n^2t}
       +C\varepsilon_n\rho_n^{-2}
\]
uniformly for $t\ge0$.

Suppose now that $\varepsilon_n/\rho_n^2\to0$. For $s\ge s_0>0$,
the first term is bounded by
\[
    C\rho_n^{-1}
    \exp\!\left(-\frac{cs_0\rho_n}{2\varepsilon_n}\right)
    \longrightarrow0,
\]
and the second term also vanishes. Finally,
\[
\begin{aligned}
    &\|w_n(\ell_ns)-(1-L_n(s))\hat w_n\| \le
    \|e_n(\ell_ns)\|
    +|\Theta_n(s)-L_n(s)|\,\|\hat w_n\|.
\end{aligned}
\]
The second term is uniformly
$O_{\mathbb P}(\varepsilon_n/\rho_n^{3/2})=o_{\mathbb P}(1)$.
This completes the proof. \qed

\subsection{Proof of Theorem~\ref{thm:underparam}}

Since $\rho_n\to1-\gamma>0$, the additional tracking condition
in Proposition~\ref{prop:tracking} holds automatically.
Moreover,
\[
    \|\hat w_n\|^2\to G_\gamma:=\frac{\gamma}{1-\gamma}.
\]
Combining nuisance tracking with core convergence gives, uniformly
on every $[s_0,S]\subseteq(0,\infty)$,
\[
    \|w_n(\ell_ns)\|^2
    \longrightarrow G_\gamma(1-\bar\Theta(s))^2.
\]
The limiting core stays strictly below one on every such
compact interval. The limiting nuisance norm is therefore
bounded away from zero there. Applying the exact test accuracy
formula proves the asserted uniform accuracy convergence.

The mismatch bound reduces to
\[
    \|e_n(t)\|\le Ce^{-ct}+C\varepsilon_n
\]
on events of probability tending to one. Given $\xi>0$,
the second term is eventually at most $\xi/2$, and the
first is at most $\xi/2$ after a deterministic time.
Thus $T^\xi(\|e_n\|)=O_{\mathbb P}(1)$.

\paragraph{Excluding crossings before the slow phase.}
Define
\[
    F_\gamma(x)=
    \Phi\!\left(\frac{x}{(1-x)\sqrt{G_\gamma}}\right),
    \quad 0\le x<1.
\]
The fast test curve tends to $F_\gamma(\theta_0)$.
Since $p>p_{\mathrm{fast}}(\gamma)$,
\[
    F_\gamma(\theta_0)<p,
    \quad \theta_0<\delta_p.
\]

We claim that, for sufficiently small fixed $s_0>0$,
there is no crossing before $\ell_ns_0$ with probability
tending to one. On every fixed interval $[0,T]$, this
follows from uniform fast-phase convergence and the strict
gap $p-p_{\mathrm{fast}}(\gamma)>0$.

For the remaining interval $[T,\ell_ns_0]$, core convergence
and continuity at $s=0$ ensure that $\theta_n(t)$ stays
arbitrarily close to $\theta_0$ when $s_0$ is small.
Meanwhile,
\[
    \|w_n(t)-(1-\theta_n(t))\hat w_n\|
    \le Ce^{-cT}+C\varepsilon_n.
\]
First choose $s_0$ small, then $T$ large. The exact accuracy
formula and $F_\gamma(\theta_0)<p$ then keep the test
accuracy strictly below $p$ throughout this interval,
with probability tending to one.

\paragraph{Attainable targets.}
Suppose $\beta<1-\delta_p$. The limiting accuracy
$F_\gamma(\bar\Theta(s))$ crosses $p$ strictly increasingly
at $s=H_{\delta_p}(\beta)$. Uniform convergence on compact
positive slow-time intervals brackets this crossing,
exactly as in the core hitting-time argument. The preceding
early-time exclusion ensures that it is the first crossing.
Hence
\[
    T_p(\mathcal A_{\mathrm{te},n})/\ell_n
    \overset{\mathbb P}{\longrightarrow}
    H_{\delta_p}(\beta).
\]

\paragraph{Unattainable targets.}
Suppose $\beta>1-\delta_p$. The global core bound gives
\[
    \sup_{s\ge0}\Theta_n(s)
    \le
    \max\{\theta_0,(1-\beta)_+\}+o_{\mathbb P}(1),
\]
whose right-hand side is strictly below $\delta_p$.
For every fixed $s_0>0$, nuisance tracking holds uniformly
for all $s\ge s_0$. The nuisance norm is then bounded
away from zero, and the exact accuracy formula gives
\[
    \sup_{s\ge s_0}\mathcal A_{\mathrm{te},n}(\ell_ns)
    \le
    F_\gamma\!\left(\max\{\theta_0,(1-\beta)_+\}\right)
    +o_{\mathbb P}(1)<p
\]
with probability tending to one. Together with the
early-time exclusion, this proves non-attainment. \qed

\subsection{Proof of Theorem~\ref{thm:edge}}

Here $\rho_n\to0$. Since
$\beta_n\sqrt{\rho_n^{-1}-1}\to b<\infty$, we have
$\beta_n\to0$. The core conclusions therefore follow
directly from Proposition~\ref{prop:tracking} with $\beta=0$.

\paragraph{Nuisance entry and stationary limits.}
The mismatch bound shows that, for every fixed $\xi>0$,
permanent entry occurs by a time of order
$\rho_n^{-2}\log\rho_n^{-1}$ with probability tending to one.
Furthermore,
\[
    \frac{\rho_n^{-2}\log\rho_n^{-1}}{\ell_n}
    =\frac{2\varepsilon_n}{\rho_n}\log\frac{1}{\rho_n}
    \longrightarrow0.
\]

To prove divergence of the entry time, fix $T<\infty$.
Fast-phase convergence gives $\theta_n(T)\to\theta_0$,
while $\|w_n(T)\|=O_{\mathbb P}(1)$ and
$\sqrt{\rho_n}\|\hat w_n\|\to1$. The triangle inequality
therefore gives
\[
    \sqrt{\rho_n}\|e_n(T)\|
    \overset{\mathbb P}{\longrightarrow}1-\theta_0>0.
\]
Thus $\Pr\{\|e_n(T)\|>\xi\}\to1$, which implies
$\Pr\{T^\xi(\|e_n\|)>T\}\to1$.

The finite-sample stationary formula is
\[
    \theta_n(\infty)
    =\left(1-\beta_n/c_n\right)_+,
    \quad
    w_n(\infty)=(1-\theta_n(\infty))\hat w_n.
\]
Since $c_n\to1$, $\beta_n\to0$, and
$\beta_n\|\hat w_n\|\to b$, the stated stationary limits follow.

\paragraph{A time after which absolute tracking holds.}
Set
\[
    t_n^\circ=\rho_n^{-2}\log^2 \rho_n^{-1},
    \quad s_n^\circ=t_n^\circ/\ell_n.
\]
The assumptions imply $s_n^\circ\to0$. The mismatch estimate
and the global core approximation give
\begin{equation}
    \sup_{s\ge s_n^\circ}
    \|w_n(\ell_ns)-(1-L_n(s))\hat w_n\|
    \overset{\mathbb P}{\longrightarrow}0.
    \label{eq:proof-edge-tracking}
\end{equation}
Indeed, the decaying term is at most
$C\rho_n^{-1}\exp[-c\log^2 \rho_n^{-1}]$, and both remaining
errors vanish under $\varepsilon_n/\rho_n^2\to0$.
The same core approximation shows that
\[
    \sup_{0\le t\le t_n^\circ}
    |\theta_n(t)-\theta_0|\to0.
\]

\paragraph{Excluding early test crossings.}
For every fixed $T$, fast-phase convergence excludes a crossing
on $[0,T]$ when $p>p_{\mathrm{fast}}(1)$. We also need to
exclude the interval $[T,t_n^\circ]$.

On an event whose probability tends to one,
$1-\theta_n(t)\ge(1-\theta_0)/2$ for $t\le t_n^\circ$.
In an eigenbasis of $B_n$, variation of constants therefore
gives, for $T\le t\le t_n^\circ$,
\[
    \|w_n(t)\|
    \ge
    \frac{1-\theta_0}{2}
    \left[
        \sum_j a_{n,j}^2
        \left(
            \frac{1-e^{-(b_{n,j}+\lambda_n)T}}
            {b_{n,j}+\lambda_n}
        \right)^2
    \right]^{1/2}
    -\|w_n(0)\|.
\]
For each fixed $T$, Gaussian quadratic-form concentration
and Mar{c}enko--Pastur convergence imply that the sum converges
in probability to
\[
    J(T)=\int\frac{(1-e^{-xT})^2}{x}\,\mu_1(\mathrm dx).
\]
This uses a bounded continuous spectral function at fixed $T$.
As $T\to\infty$, $J(T)\to\infty$, because the density of
$\mu_1$ is asymptotic to $1/(\pi\sqrt{x})$ at zero.
Choose a fixed $T$ large enough that the preceding lower
bound exceeds $2/z_p$ with probability tending to one.
Since $\theta_n(t)\le2$ on the same event, the test accuracy
then stays below $p$ on $[T,t_n^\circ]$.

\paragraph{Logistic behavior on the longer clock.}
With probability tending to one, put
\[
    q_n=c_n-\beta_n>0,\quad
    \kappa_n=1-\beta_n/c_n>\theta_0,\quad
    A_n=\kappa_n/\theta_0-1>0.
\]
The explicit comparison trajectory is
\[
    L_n(s)=\frac{\kappa_n}{1+A_ne^{-q_ns}},
\]
and so
\begin{equation}
    1-L_n(s)
    =\frac{\beta_n}{c_n}
    +\frac{\kappa_nA_ne^{-q_ns}}{1+A_ne^{-q_ns}}.
    \label{eq:proof-edge-deficit}
\end{equation}
Here $q_n\to1$, $\kappa_n\to1$, and
$A_n\to(1-\theta_0)/\theta_0>0$.

Now, fix $a\in(0,1/2)$ and define
\[
    s_n^-=(1/2-a)\log\rho_n^{-1},
    \quad
    s_n^+=(1/2+a)\log\rho_n^{-1}.
\]
Equation~\eqref{eq:proof-edge-deficit} gives
\[
    (1-L_n(s_n^-))\|\hat w_n\|\to\infty,
    \quad
    (1-L_n(s_n^+))\|\hat w_n\|\to b
\]
in probability. For the first assertion, the decaying term
has size $\rho_n^{1/2-a+o_{\mathbb P}(1)}$ before multiplication
by $\|\hat w_n\|\simeq_{\mathbb P}\rho_n^{-1/2}$.
For the second, its size is
$\rho_n^{1/2+a+o_{\mathbb P}(1)}$, while the stationary term
converges to $b$ after that multiplication.

\paragraph{Attainable targets.}
Suppose $b<z_p^{-1}$. Since $L_n$ is increasing,
the first preceding limit and \eqref{eq:proof-edge-tracking}
imply that the nuisance norm diverges uniformly over
$s\in[s_n^\circ,s_n^-]$. Thus test accuracy stays below
$p$ there with probability tending to one.
Together with early-time exclusion, this gives
\[
    T_p(\mathcal A_{\mathrm{te},n})>\ell_ns_n^-
\]
with probability tending to one.

At $s_n^+$, core convergence through the global approximation
gives $\theta_n(\ell_ns_n^+)\to1$, and
\eqref{eq:proof-edge-tracking} gives
$\|w_n(\ell_ns_n^+)\|\to b$.
The test accuracy therefore tends to $\Phi(1/b)>p$,
with the value interpreted as one when $b=0$.
Consequently,
\[
    T_p(\mathcal A_{\mathrm{te},n})\le\ell_ns_n^+
\]
with probability tending to one. Letting $a\downarrow0$ proves
\[
    T_p(\mathcal A_{\mathrm{te},n})
    \simeq_{\mathbb P}\frac{\ell_n}{2}\log\frac{1}{\rho_n}.
\]
Finally, direct substitution in \eqref{eq:Tdelta} gives
\[
    H_{\delta_{n,p}}(\beta_n)
    =\frac12\log\frac{1}{\rho_n}+O(1),
\]
which proves the other equivalent in the statement.

\paragraph{Unattainable targets.}
Suppose $b>z_p^{-1}$. From \eqref{eq:proof-edge-deficit},
\[
    (1-L_n(s))\|\hat w_n\|
    \ge\frac{\beta_n}{c_n}\|\hat w_n\|
    =b+o_{\mathbb P}(1)
\]
uniformly for all $s\ge0$. Combining this with
\eqref{eq:proof-edge-tracking} and
$\sup_s\Theta_n(s)\le1+o_{\mathbb P}(1)$ gives
\[
    \sup_{s\ge s_n^\circ}
    \mathcal A_{\mathrm{te},n}(\ell_ns)
    \le\Phi(1/b)+o_{\mathbb P}(1)<p
\]
with probability tending to one. Early-time exclusion
handles the remaining interval, proving non-attainment. This completes the proof of the theorem. \qed
\section{Experimental Validation}
\label{sec:experiments}

In this section, we provide extended experimental results under gradient flow (GF) analyzed in the main text.
Furthermore, we show that our theoretical predictions directly extend to gradient descent via rescaled time $\tau = k \cdot \eta$, where $\eta$ denotes the step size and $k$ the iteration index.

\subsection{Setup}
\label{sec:setup}
All experiments are run on a synthetic dataset generated according to the distribution defined in \eqref{eq:distro}.
Continuous-time trajectories are obtained by numerically integrating the exact ODEs \eqref{eq:exact}.
For discrete-time gradient descent, student models (\ref{eq:toy-model}) are trained using vanilla gradient descent with a learning rate of $\eta=0.5$.
Core feature weights $u$ and $v$ are initialized following Assumption~\eqref{eq:assumuv}.
Nuisance feature weights $w_0$ are initialized using PyTorch's default uniform initialization and subsequently normalized to unit norm ($\|w\|=1$).

\subsection{Results}
We empirically validate our main Theorem \ref{thm:big} across both gradient flow and gradient descent.

\paragraph{Grokking trajectories ($\lambda\approx1)$}
Figure~\ref{fig:hard_interp_app} presents experimental results in the hard interpolation regime. Note that the signal strength remains small across all experiments to satisfy $\eps/\rho_n^2\to0$, as required at the interpolation edge. As predicted by our theory, regularization in this regime prevents from reaching the target test accuracy.

\paragraph{Gradient flow and gradient descent}
Figure~\ref{fig:underparameterized_grokking_ode_vs_gd} compares the temporal evolution of accuracy and the core feature $\theta$ under gradient descent (top) and gradient flow (bottom).
Both experiments use $n = 400$ samples, $d = 360$ nuisance features ($\gamma = d/n = 0.9$), and initial core feature scale $\theta_0 = 10^{-1}$.
The close agreement confirms that discrete gradient descent trajectories converge to the corresponding gradient flow limit under the time rescaling.

\paragraph{Phase portrait}
Figure \ref{fig:experiment_underparams_phase} presents the phase portrait of $\|w\|$ against $\theta$ in the under-parametrized regime. Notice the fast relaxation of the fast variable $W$ and the slow evolution along  the neighborhood of the slow manifold, following Fenichel theory. Notably, when grokking fails to occur, the phase trajectory shifts toward the left.

\paragraph{Logarithmic scaling of grokking time}
Figure~\ref{fig:plateau_scaling_gd} illustrates the normalized grokking time $T_p / \ell_n$ as a function of log-inverse core feature initialization $\log(1/\theta_0)$ under gradient descent ($n = 400, d = 360, \varepsilon = 10^{-2}, \lambda = 0$). The empirically measured hitting times $T_p/\ell_n$ scale linearly with $\log(1/\theta_0)$ and closely follow the theoretical prediction $H_{\delta}(\beta)$ \eqref{eq:H}. 
This demonstrates that the logarithmic plateau time derived for gradient flow holds identically for discrete gradient descent. Figure~\ref{fig:edgy} and Figure~\ref{fig:plateau_n=10000} present further experiments with larger sample size $n$ and feature dimension $d$.

\begin{figure}[H]
    \centering
     \begin{subfigure}{\textwidth}
        \centering
        \includegraphics[width=0.9\linewidth]{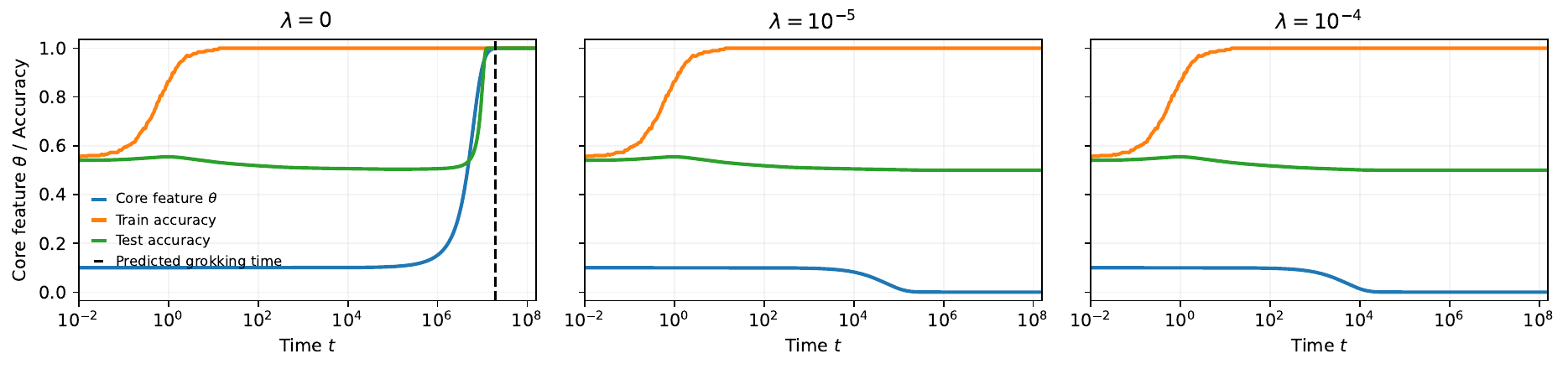}
        \caption{\textbf{Vanishing SNR $\eps=10^{-5}$.}}
      \end{subfigure}
      \medskip
     \begin{subfigure}{\textwidth}
        \centering
        \includegraphics[width=0.9\linewidth]{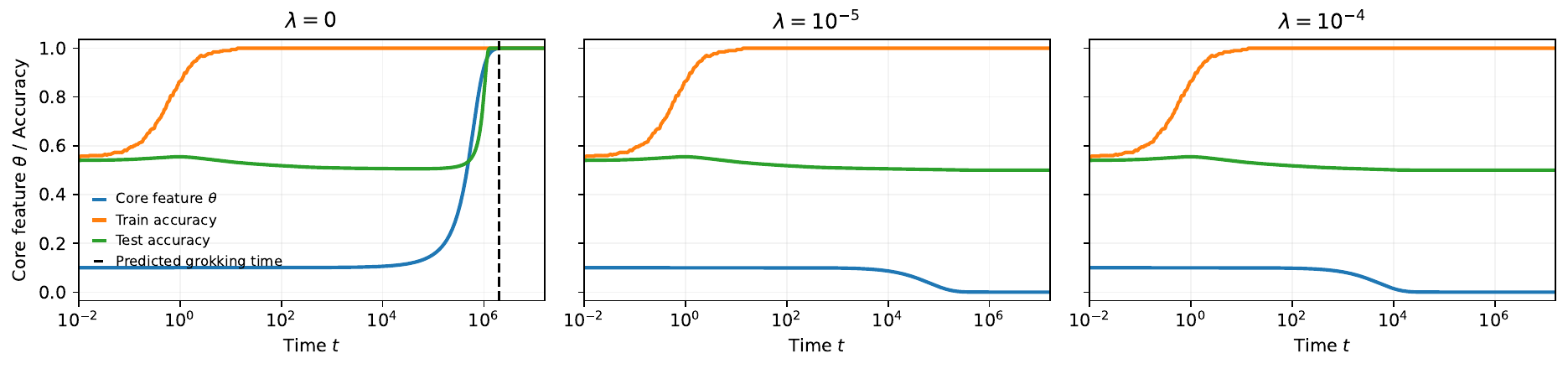}
        \caption{\textbf{Very weak SNR $\eps=10^{-4}$.}}
      \end{subfigure}
      \medskip
     \begin{subfigure}{\textwidth}
        \centering
        \includegraphics[width=0.9\linewidth]{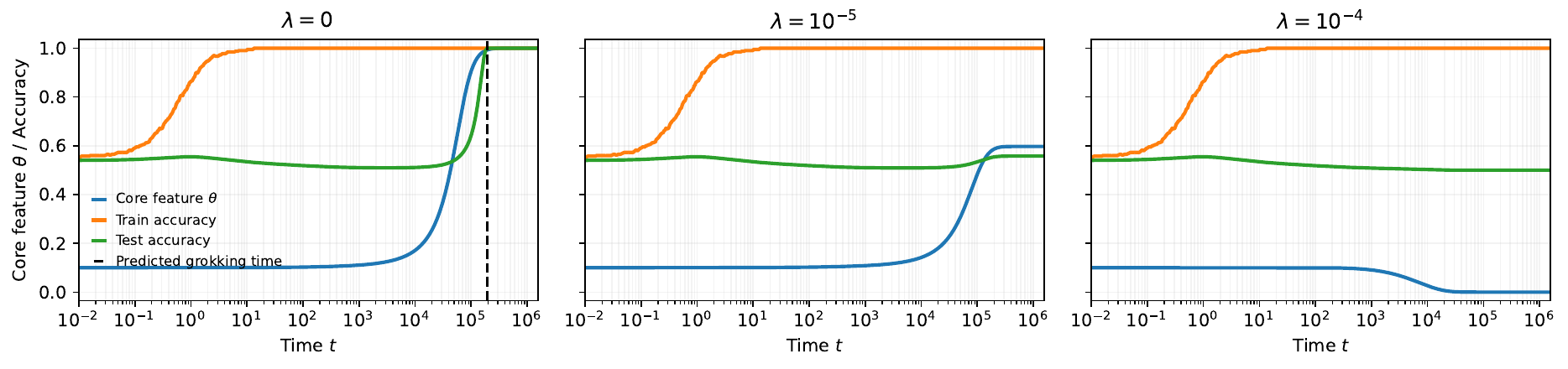}
        \caption{\textbf{Weak SNR $\eps=10^{-3}$.}}
      \end{subfigure}
    \caption{\textbf{Grokking at hard interpolation $\gamma=1$}. Here we use $n=400$ samples and $d=395$ nuisance features ($d/n\approx1$). This is an extension of experiments in Figure \ref{fig:experiment_hard_interp}, using regularized models $\lambda\in\{10^{-5}, 10^{-4}\}$.
    Here, critical threshold is $1/z_p=0.60$ and $b>1/z_p$ in last two columns.
    For fixed $n$ and $d$, $\beta_n/\sqrt{\rho_n^{-1}-1}\to b$ increases as the regularization increases, making grokking harder to happen.}
    \label{fig:hard_interp_app}
\end{figure}

\begin{figure}[H]
    \centering
    \begin{subfigure}{\textwidth}
        \centering
        \includegraphics[width=0.9\linewidth]{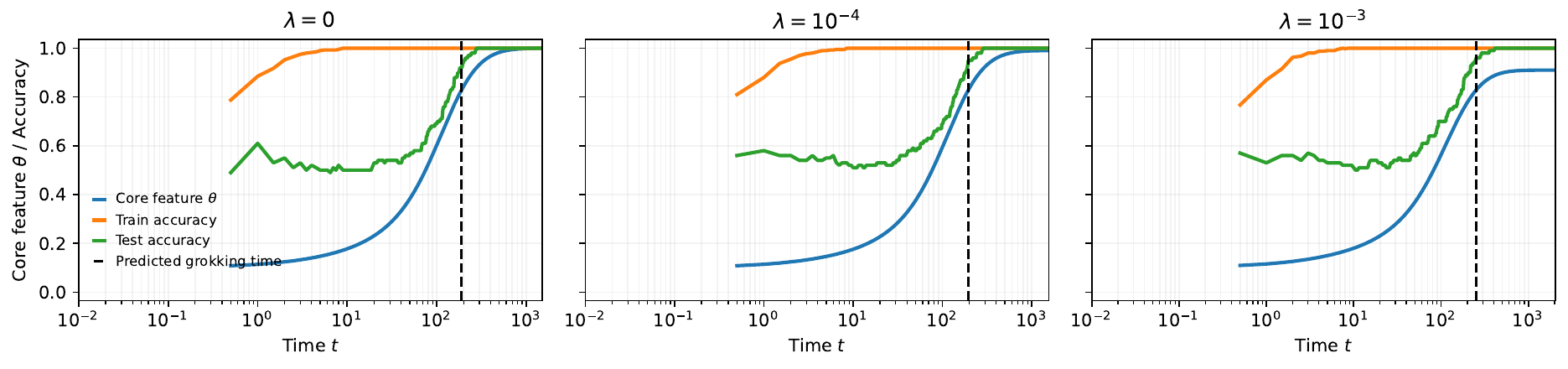}
        \caption{Gradient descent (discrete dynamics)}
      \end{subfigure}
      \medskip
    \begin{subfigure}{\textwidth}
        \centering      
        \includegraphics[width=0.9\linewidth]{figures/ode_n=400_d=360_theta0=5.00e-02_p=0.95_seed=15_acc_theta_epsilon=1.00e-01.pdf}
        \caption{Gradient flow (continuous dynamics)}
    \end{subfigure}
    \caption{\textbf{Grokking in the under-parameterized regime ($\gamma < 1$) for gradient descent and continuous gradient flow.} 
    Both experiments use $n = 400$ samples, $d = 360$ nuisance features ($\gamma = d/n = 0.9$), and initial core feature scale $\theta_0 = 10^{-1}$. Signal strength is $\varepsilon=10^{-1}$.
    (a) Gradient Descent with rescaled time. 
    (b) Continuous Gradient Flow ODE. 
    The close match confirms that the theoretical continuous-time analysis accurately captures discrete-time dynamics.}
    \label{fig:underparameterized_grokking_ode_vs_gd}
\end{figure}

\begin{figure}[H]
    \centering
    \includegraphics[width=.9\linewidth]{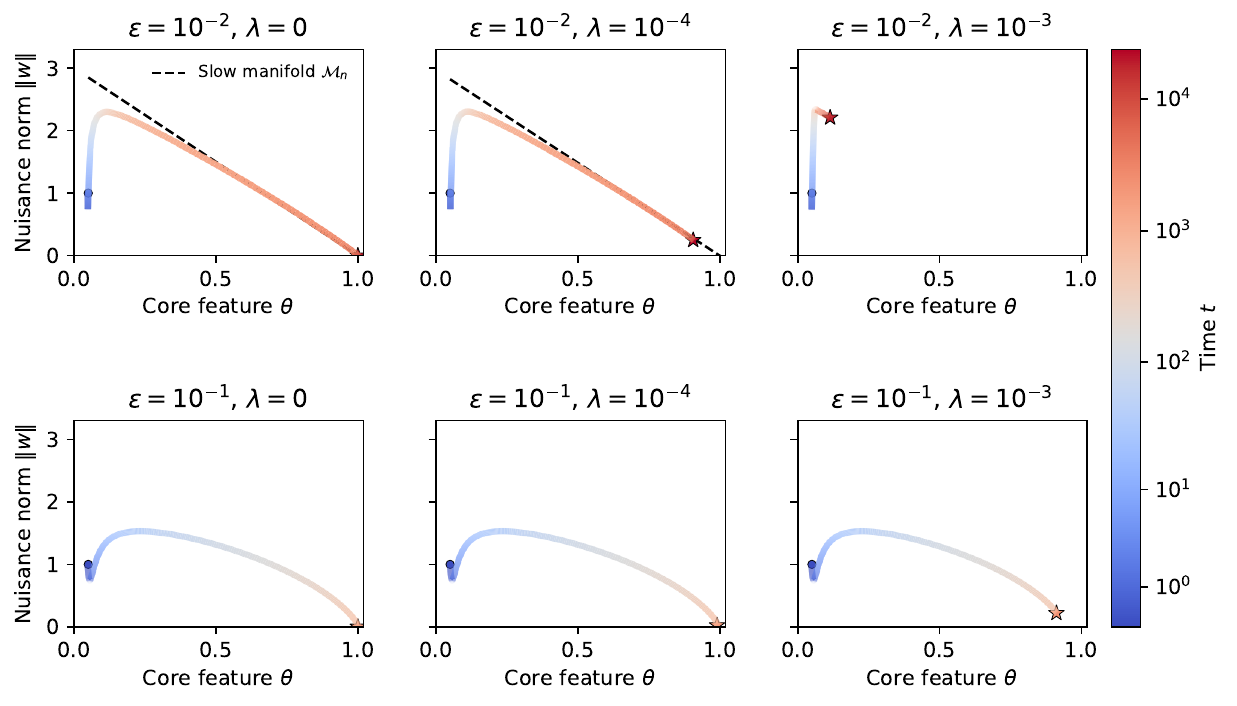}
    \caption{\textbf{Phase portrait $\|w\|$ against $\theta$ in under-parametrized model ($\gamma=0.9$)}.
    Experiments use $n=400$ samples, and $d=360$ nuisance features.
    Results illustrate phase portraits from experiments in Figure \ref{fig:experiment_underparams}.
    For small SNR $\eps=10^{-2}$ (top row), dynamics exhibit a fast inital relaxation in time scale of $O(1)$ (visible as a near-vertical trajectory at the beginning of training), after which nuisance feature weights $w$ are trapped in slow manifold $\mathcal M_0$, directly validating Theorem \ref{thm:big} via Fenichel theory. Notice that moderate SRN $\eps=10^{-1}$ (bottom row) accelerates grokking, yielding a reduced hitting time, in exact agreement with our analytical $T_p \simeq H_p(\beta) \ell_n$.}
    \label{fig:experiment_underparams_phase}
\end{figure}

\begin{figure}[H]
    \centering
    \includegraphics[width=.48\linewidth]{figures/normalized_hitting_time_n=400_d=360_epsilon=0.001_lambda=0.0_p=0.95.pdf}
    \includegraphics[width=.48\linewidth]{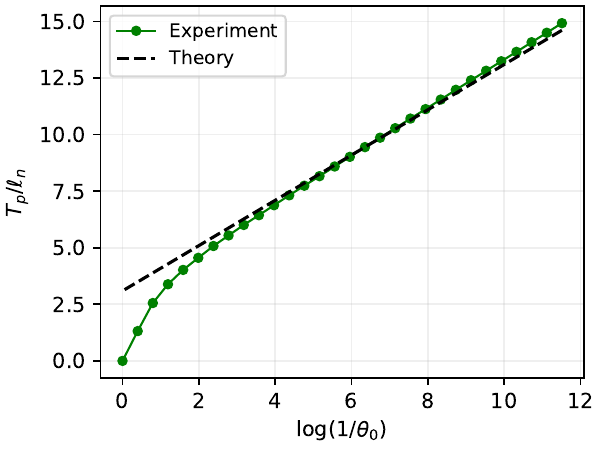}
    \vspace{-.25cm}
    \caption{\textbf{Logarithmic scaling of grokking time w.r.t core feature initialization $\theta_0$ under Gradient Descent.}
    Comparing gradient flow (\textbf{left}) and gradient-descent (\textbf{right}). The settings are identity to Figure \ref{fig:plateau_scaling} of the main text. GD is run with setup as described in Section \ref{sec:setup}. }
    \label{fig:plateau_scaling_gd}
\end{figure}

\begin{figure}[H]
    \centering
    \includegraphics[width=0.48\linewidth]{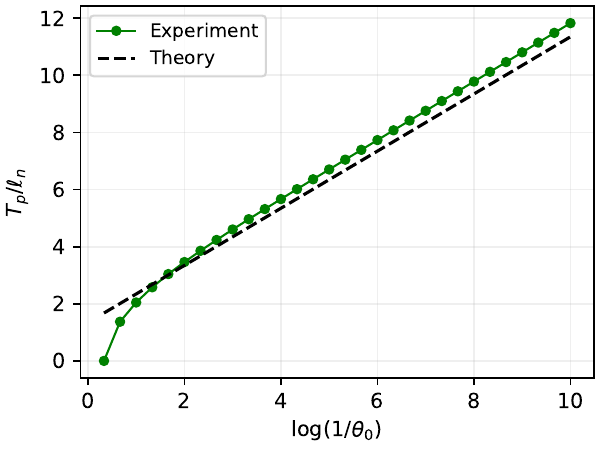}
    \includegraphics[width=0.48\linewidth]{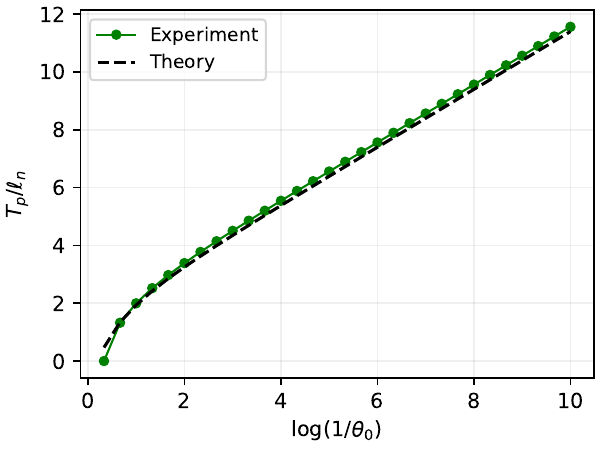}
    \caption{\textbf{Plateau length as a function of initialization scale.} For these experiments, the signal strength is set to $\varepsilon=10^{-2}$ and the weight decay is turned off ($\lambda=0$). In each case, the straight broken line (theory) represents $y=x+b$, where $x=\log1/\theta_0$ and the intercept $b$ is an explicit function of the target test accuracy (here $p=0.95$), but not on the initialization $\theta_0$, etc. The grokking time $T_p$ is measured empirically. The scale $\ell_n$ is as defined in \eqref{eq:biglimit}. \textbf{Left:} Under-parametrized regime $\gamma<1$. We use $n=400$ and $d=360$ as in Figure \ref{fig:plateau_scaling}. The theoretical curve is given by \eqref{eq:underparam-Tdelta}. \textbf{Right:} Near the hard interpolation edge $\gamma=1$. Here, we use $n=2000$ and $d=1800$. For the theory curve, we use \eqref{eq:edge-Tdelta}. As in \eqref{eq:underparam-Tdelta}, notice the perfect match between experiment and our theoretical predictions from Section \ref{sec:plateau-lengths}.}
    \label{fig:edgy}
\end{figure}

\begin{figure}[H]
    \centering
    \includegraphics[width=0.5\linewidth]{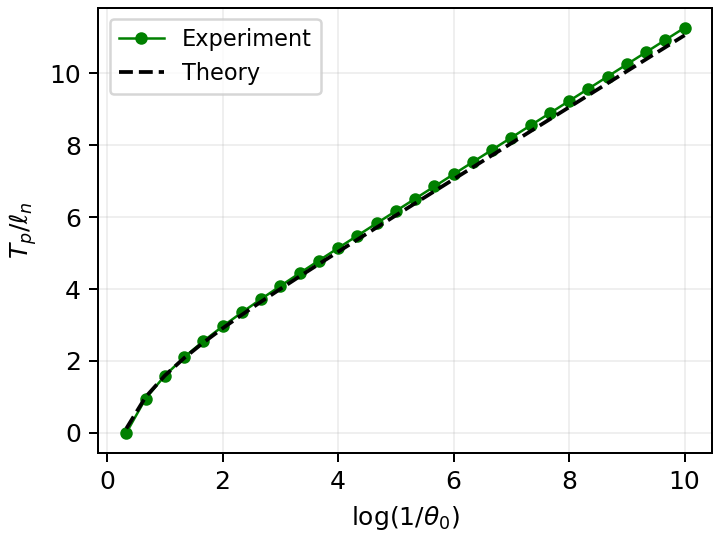}
    \caption{\textbf{Plateau length in under-parametrized regime $\gamma<1$.}
    Here $n=10000$ samples and $d=8000$, making $k=n-d$ larger.
    Signal strength is $\eps=10^{-2}$ and regularization is odd. Under the limit \eqref{eq:biglimit}, the measured normalized grokking time is close to the theory as expected.}
    \label{fig:plateau_n=10000}
\end{figure}

\end{document}